\pdfoutput=1

\documentclass[10pt]{article} %

\usepackage[accepted]{rlj} %

\usepackage{amssymb}            %
\usepackage{mathtools}  
\usepackage{algpseudocode}
\usepackage{booktabs}
\usepackage{algorithm}
 \usepackage{amsmath}%
 \usepackage{amsthm}
\usepackage{mathrsfs}           %
\usepackage{graphicx}           %
\usepackage{subcaption}         %
\usepackage{url}                %
\usepackage{wrapfig}
\usepackage{hyperref}

\newcommand{\States}{\mathcal{S}}
\newcommand{\Actions}{\mathcal{A}}
\newcommand{\Reward}{\mathcal{R}}
\newcommand{\Transition}{\mathcal{P}}

\newcommand{\fullbuffer}{\mathcal{D}}
\newcommand{\corebuffer}{\mathcal{D}_c}
\newcommand{\recencybuffer}{\mathcal{D}_r}
\newcommand{\expectile}{u}

\newtheorem{proposition}{Proposition}
\newtheorem{definition}{Definition}
\newtheorem{assumption}{Assumption}
\newtheorem{theorem}{Theorem}
\newtheorem{corollary}{Corollary}
\newtheorem{lemma}{Lemma}

\title{Endpoint Replay: Compressing the Recency Buffer in Deep Reinforcement Learning}

\setrunningtitle{Endpoint Replay: Compressing the Recency Buffer in Deep Reinforcement Learning}

\author{Parham Mohammad Panahi\textsuperscript{1,2}, Armin Ashrafi\textsuperscript{1,2}, Haoyu Du\textsuperscript{1,2}, Andrew Patterson\textsuperscript{1,2}, Martha White\textsuperscript{1,2,3}, Adam White\textsuperscript{1,2,3}}

\coverPageAuthor{Parham Mohammad Panahi, Armin Ashrafi, Haoyu Du et al.}

\emails{\{parham1,ashrafi2,du2,ap3,whitem,amw8\}@ualberta.ca}

\affiliations{
$^{1}$\textbf{Department of Computing Science, University of Alberta, Canada}\\
$^{2}$\textbf{Alberta Machine Intelligence Institute (Amii)}\\
$^{3}$\textbf{Canada CIFAR AI Chair}\\
}

\contribution{
    We introduce Endpoint Replay, a buffer-compression algorithm that pairs a small recency buffer with a coreset of chained $n$-step transitions, updated with Expectile Sarsa. With 10x to 50x less storage, it matches the performance of a standard large recency buffer in Pinball and across twelve Atari 2600 games, and outperforms equal-sized recency buffers (with 1-step or 10-step targets), reservoir-sampling coresets, and MeDQN. Ablations show anchored bootstrap states and the expectile loss drive this performance.
    }
    {
    We are not the first to shrink the buffer: continual RL has compressed buffers to improve tracking~\citep{isele2018selective,chaudhry2019tiny}, MeDQN regularizes toward the target network to use an extremely small buffer~\citep{lan2022memory}, and streaming methods remove the buffer entirely~\citep{vasan2024deep,elsayed2024streaming}. Our goal differs: retaining the benefits of a large buffer at a fraction of its size, not eliminating replay.
    }

\contribution{
    We identify the previously unrecognized issue of \emph{unanchored bootstrap targets}: a coreset of isolated transitions bootstraps off states (and actions) that are never updated by any transition in the buffer, yielding inaccurate targets. We demonstrate the issue in controlled prediction and control experiments, and show that Endpoint Replay’s chained $n$-step construction reduces bootstrap target error where interval and reservoir sampling degrade.
    }
    {
    Any coreset of isolated transitions is susceptible, including distance-based prototype selection from classical coreset construction~\citep{Mirzasoleiman2020Coresets} and prior continual RL buffers~\cite{isele2018selective,yang2024augmenting}. The issue parallels bootstrapping on out-of-sample actions in offline RL~\citep{xiao2023sample}.
    }

\contribution{
    We formalize n-step Expectile Sarsa, which mitigates the pessimistic bias of $n$-step targets computed from older, more suboptimal policies. We define the $n$-step expectile action-values, prove the associated Bellman operator converges for $n \geq 1$, and prove policy iteration converges—to the optimal policy when dynamics are deterministic. Unlike the expected return, the fixed point depends on $n$, and we prove an ordered relationship on the n-step expectile values.
    }
    {
    Expectiles were used in offline RL by IQL~\citep{kostrikovoffline} to avoid out-of-distribution actions. Convergence with 1-step bootstrapping is known for the robust setting ($\tau < 0.5$); to our knowledge, $n > 1$ has not been analyzed. 
    }


\summary{Nearly every off-policy deep reinforcement learning (DRL) algorithm uses the replay recipe established by DQN: a FIFO buffer of the last one million transitions, sampled uniformly. This paper asks whether that buffer can be compressed by an order of magnitude or more without sacrificing performance. We show that the natural approach-keeping a coreset of isolated, representative transitions—fails for a previously unrecognized reason: the resulting bootstrap targets are \emph{unanchored}, querying values at states and actions that are never themselves updated, which corrupts the value function. We introduce Endpoint Replay, which compresses experience as it leaves a small recency buffer into a chain of connected n-step transitions, so every bootstrap target remains anchored, and which uses an expectile Sarsa update to counteract the pessimistic bias of n-step targets computed from older policies. We prove that the resulting n-step expectile updates are sound, and that policy iteration converges to the optimal policy under deterministic dynamics. Empirically, Endpoint Replay with 10x to 50x less storage matches large-buffer performance in Pinball and across twelve Atari 2600 games, and outperforms equal-sized buffers, reservoir-sampling coresets, and MeDQN.}

\begin{document}

\makeCover  %
\maketitle  %

\begin{abstract}
Experience replay remains one of the most practical and useful algorithmic tools in the deep reinforcement learning (DRL) toolbox. Aside from the limited success of prioritized replay and specialized approaches for large asynchronous systems, most DRL algorithms make use of a large, uniformly sampled recency buffer---even the size, one million, remains unchanged. Could we store less data, reduce redundancy, or more effectively chain experience together to speed up value propagation and still retain the performance of large buffers? In this paper, we investigate a simple compression approach that stores representative transitions derived from the end-points of a chain of connected $n$-step sequences. By curating these end-points in a smaller recency buffer, our method maintains an effective memory horizon comparable to a standard large buffer while requiring an order of magnitude less storage. Through empirical evaluation, we demonstrate that this approach prevents the systematic bias inherent in naive compression strategies and matches the performance of traditional large buffers in the Pinball environment and the Atari 2600 benchmark.\footnote{Implementation and experiments code is available on \href{https://github.com/panahiparham/endpoint-replay}{github.com/panahiparham/endpoint-replay}.}
\end{abstract}

\section{Introduction}
\label{sec:intro}

The idea of storing and reusing raw recent experience, as in Experience Replay (ER), remains one of the most useful and practical algorithmic tools in Deep Reinforcement Learning (DRL). Replay is a key component of almost every successful off-policy DRL algorithm, playing multiple roles including: de-correlating samples to improve network training and avoiding catastrophic forgetting~\citep{mnih2015human}, retaining old data which is critical for hard exploration tasks~\citep{fedus2020revisiting}, propagating error and sparse rewards~\citep{panahiinvestigating}, and learning many things in parallel~\citep{andrychowicz2017hindsight,mnih2016asynchronous,schaul2015universal,jaderberg2017reinforcement}. Widely used implementations of pretty much every popular RL algorithm, like, SAC, MPO, and even DreamerV3, all use basically the same recipe established in the original DQN paper: a FIFO buffer, sampled uniformly, with a replay ratio less than one, of size one million.   

There have been many attempts to improve on this basic recipe.
Motivated by hippocampal replay in the brain~\citep{igata2021prioritized}, the success of prioritization in tabular planning~\citep{peng1993efficient, moore1992memory}, and backward replay~\citep{lin1991programming}, one line of work explores different weighting schemes based on error/surprise   prioritization~\citep{schaul2016prioritized,panahiinvestigating,kumarintrospective}, state similarity~\citep{sun2020attentive,hong2022topological}, and temporal order~\citep{lee2019sample}. Another line of work focuses on increasing the replay ratio; the number of gradient updates performed per environment step~\citep{dsample,kang2025forget,maheshwari2025altnet}. Finally, instead of storing a buffer, a generative model could be used to simulate synthetic transitions from the real world~\citep{lu2023synthetic,wang2025prioritized}, which blurs the lines between replay and model-based RL~\citep{pan2018organizing,van2019use}. All the above variants keep the idea that the buffer should be large. In conventional RL, large buffers are almost always preferred in terms of performance~\citep{DBLP:journals/corr/abs-1712-01275,fedus2020revisiting}, though some work has shown that small buffers can be beneficial ~\citep{fedus2020revisiting,DBLP:journals/corr/abs-1712-01275,chaudhry2019tiny,kang2025forget}. Motivated by necessity, work in continual RL has investigated the utility of smaller buffers to mitigate the impact of non-stationarity and task switches~\citep{isele2018selective,rolnick2019experience,chaudhry2019tiny,chaudhry2021using,yang2024augmenting,zheng2024selective}. Others have explored extremely small replay buffers~\citep{lan2022memory}, or using no buffers at all ~\citep{elsayed2024streaming,vasan2024deep}. Our goal is to reduce the size of the replay buffer maintaining the benefits of large buffers as outlined above, not to eliminate them completely.

In this paper, we present a novel approach to reducing the replay buffer size while maintaining performance. The idea is to maintain a small recency buffer and a slightly larger buffer that, like prior work~\citep{isele2018selective} stores transitions representative of a large buffer, called a {\em coreset buffer}~\citep{Mirzasoleiman2020Coresets,pmlr-v227-zheng24a,pan2018organizing}. Together these two buffers can be 10 to 50 times the size of a conventional buffer. The idea of using multiple buffers is not new~\citep{anand2026permanent,yang2024augmenting,sinha2022experience,mohd2025dual,futuhi2024etgl,hester2018deep,wurman2022outracing}, but care must be taken in how the coreset buffer is constructed and sampled. 
Our coreset buffer is always constructed from the data as it falls out of the smaller recency buffer. From that we construct a chain of connected jumpy $n$-step transitions, dropping the transitions in between. This step is critical because a coreset buffer constructed from disconnected representative transitions (as done in prior work ~\citep{isele2018selective,yang2024augmenting}) would create temporally isolated bootstrap target values. These {\em unanchored values} provide bad bootstrap targets, significantly impacting performance as our experiments show. We call our approach {\em Endpoint Replay}. Our method differs from MeDQN~\citep{lan2022memory} which also attempts to reduce the size of the replay buffer, but does so by regularizing the current value estimates to the target network. This approach can slow learning, but more importantly requires randomly generating states which could be invalid, hurting performance as our experiments show. 

We provide experiments to illustrate the severity of unanchored states and show that Endpoint Replay is competitive with both large buffer and small buffer agents across multiple environments including 12 Atari 2600 games~\citep{machado2018revisiting}. In particular, we demonstrate how unanchored next-states in the buffer can cause inaccurate bootstrap target values in prediction and control and how Endpoint replay mitigates this inaccuracy. Comparing Endpoint replay to MeDQN, and other baselines in the Pinball environment ~\citep{konidaris2009skill}, we find our approach comparable to a much larger recency buffer even with a 50 times reduction in buffer size. Experiments in Atari 2600 yield similar results, while ablations across both Pinball and Atari show the importance of both anchoring states and our expectile update~\citep{kostrikovoffline}. These results highlight that Endpoint replay achieves our goal of maintaining the performance of large buffers while achieving significant compression across two very different environments.

\section{Problem Formulation}
\label{sec:background}

The agent-environment interaction is formulated as a Markov Decision Processes (MDP) $\mathcal{M}=\left(\States, \Actions, \Transition, \Reward, \gamma \right)$. $\States$ is the state space and $\Actions$ is the action space. The transition function $\Transition: \States \times \Actions \rightarrow \Delta(\States)$ describes the probability of transitioning from a state action pair to another state.
The reward function is defined as $\Reward: \States \times \Actions \times \States \rightarrow \mathbb{R}$ and discount factor is $\gamma \in [0, 1)$, which is set to zero at a terminal state \citep{white2017unifying} in the episodic setting. The goal is to continually improve the agent's policy, $\pi: \States \rightarrow \Delta(\Actions)$, according to the discounted sum of the future reward, $G_t \doteq r_{t+1} + \gamma r_{t+2} + \gamma^2 r_{t+3} + \dots$, called the {\em return}. To do so, many methods learn an action-value  function $Q_\theta$, parameterized by $\theta\in\mathbb{R}^d$, to estimate the expected return under $\pi$: $\mathbb{E}_\pi[G_t|S_t{=}s, A_t{=}a]$, starting from state-action pair $(s,a)$ and taking actions according to $\pi$. 

In this paper, we particularly examine algorithms  based on Deep Q-networks (DQN) \citep{mnih2015human}. DQN stores an experience replay buffer of the agent's experience and does minibatch gradient descent updates on a transition $(s,a,s',r, \gamma)$ using the  update 
\begin{equation*}
\Delta \theta \doteq (r + \gamma \max_{a' \in \mathcal{A}} Q_{\theta^-}(s', a') - Q_{\theta}(s, a)) \nabla Q_{\theta}(s, a), 
\end{equation*}
where $Q_{\theta^-}(s', a')$ is called a \emph{target network} that changes more slowly than $Q_{\theta}$, providing a more stable target. 
Note that $\gamma$ is included in the transition to indicate if $s'$ is terminal, in which case $\gamma = 0$ to ensure the update is $r  \nabla Q_{\theta}(s, a)$ at termination.  
This update can suffer from maximization bias in the bootstrap term $\max_{a' \in \mathcal{A}} Q_{\theta}(s', a')$ \citep{hasselt2010double}, ameliorated by a modified update called Double DQN \citep{van2016deep} :
\begin{equation*}
\Delta \theta \doteq (r + \gamma Q_{\theta^{-}} (s_{t+1}, \text{arg}\max_{a} Q_{\theta}(s_{t+1}, a)) - Q_{\theta}(s, a)) \nabla Q_{\theta}(s, a) .
\end{equation*}
A key question for these algorithms is how to curate the experience replay buffer. It can be impractical to store the experience from the beginning of agent learning. Further, using all of the data for the buffer may not be desirable, because less useful transitions from early learn may dominate the buffer and handling and using such a large buffer can be computationally inefficient. A common choice is to use a relatively large recency buffer, such as storing the last million samples, and dropping any older samples. Such a buffer, however, can still be quite large. There have been some attempts to compress this large recency buffer into a smaller set of transitions called a \emph{coreset}, for example by using reservoir sampling~\citep{isele2018selective}. This is the key question we tackle in our work: can we significantly compress a large recency buffer, to 10\% or even smaller of its size, and maintain comparable performance? 

\section{Endpoint Replay}
\label{sec:end-point}

In this section, we introduce a simple approach to compress a large recency buffer, which we call \emph{Endpoint replay}. We first highlight that using a coreset with non-contiguous transitions selected from across the large recency buffer suffers from a previously unrecognized issue of \emph{unanchored bootstrap targets}. We show that a simple approach based on $n$-step targets ameliorates this issue. We then discuss the complete algorithm, which includes both a coreset of $n$-step transitions and a smaller recency buffer as well as a modified (expectile) loss for the updates from the coreset to reduce bias from using $n$-step returns. 

\newcommand{\nsamples}{m}

\subsection{The unanchored bootstrap target issue}\label{sec:anchor-pred}

Let us start by considering an on-policy prediction setting. We have a policy $\pi$ that generates the dataset $\fullbuffer = \{(s_0, a_0, r_1, \gamma_1, s_1, a_1), \ldots, (s_{\nsamples - 1}, a_{\nsamples-1}, r_\nsamples, \gamma_\nsamples,s_\nsamples, a_\nsamples)$\} of $\nsamples$ transitions (with many episodes). We can think of this dataset as our replay buffer, and consider algorithms to compress it into a coreset buffer $\corebuffer$. Many algorithms, such as using reservoir sampling or selecting representative samples, will create isolated transitions $(s_t,a_t,r_{t+1},s_{t+1},a_{t+1})$ where the next transition $(s_{t+1},a_{t+1},r_{t+2},s_{t+2},a_{t+2})$ is not in the coreset buffer $\corebuffer$. For example, a simple algorithm that just selects every $n$th point will have $\corebuffer = \{(s_0, a_0, r_1, \gamma_1, s_1), (s_n, a_n, r_{n+1}, \gamma_{n+1}, s_{n+1}) \ldots \}$. By spreading out these points, we are more likely to cover the space, but run into a previously unrecognized issue: \emph{unanchored bootstrap targets}.

The issue is that we may never update the action-value for $(s_{t+1}, a_{t+1})$ used to bootstrap these isolated transitions. When we use the full buffer $\fullbuffer$ we are always updating the values for $s_{t+1}, a_{t+1}$, which then provide a reasonable bootstrap target for $s_{t}, a_{t}$. When we subsample to create isolated transitions, the value for $s_{t+1}, a_{t+1}$ is no longer anchored by updates and instead these values will be shifted by generalization. If there are similar states in the coreset, then this issue may not arise. Otherwise, these estimates could become arbitrary, or—more likely—converge to values similar to those of a nearby state $s_t$. This issue resembles the out-of-sample action issue seen in offline learning, where specialized algorithms have been developed to avoid bootstrapping on actions not well-covered for a state in the dataset \citep{xiao2023sample}. Yet, this issue has not been recognized when using coresets in the online RL setting, and is even more problematic because isolated transitions make both the next state and action out-of-sample. 

We propose a simple fix for this problem: $n$-step returns. We can still compress the size $\nsamples$ dataset $\fullbuffer$ to a coreset $\corebuffer$ of size $\nsamples/n$ using jumpy transitions that aggregate the reward for $n$ steps between states: for $g_{t,n} \doteq \sum_{i=0}^{n-1} \gamma^{i} r_{t+i+1}$, we store $(s_{t},a_{t},g_{t,n},\gamma^n, s_{t+n},a_{t+n})$. Our coreset buffer is  
\begin{equation*}
    \corebuffer = \{(s_0, a_0, g_{0,n}, \gamma^n, s_{n}, a_{n}), (s_{n}, a_n, g_{n,n}, \gamma^n, s_{2n}, a_{2n}), (s_{2n}, a_{2n}, g_{2n,n}, \gamma^n, s_{3n}, a_{3n}) \ldots \}.
\end{equation*}
Now when we update with a transition $(s, a, g, \gamma_n, s',a') \in \corebuffer$, the target $g + \gamma_n q(s',a')$ is an $n$-step target and further every $s',a'$ we bootstrap from is itself in the coreset as the beginning of another $n$-step tuple that gets updated. This simple change to chain $n$-step transitions now allows the same number of transitions to be stored as for the isolated, unanchored bootstrap targets, but with anchored bootstrap targets. 

We demonstrate that this conceptual problem does indeed manifest in practice. Again, we start in prediction where conceptually this is already a clear problem and then also demonstrate it persists in control in Section \ref{sec:anchor-control}. 
We use the PinBall environment~\citep{konidaris2009skill,lo2024goal} for this experiment. We first train a Double DQN (DDQN) agent with $\varepsilon$-greedy exploration until it reaches good online performance. We use the $\varepsilon$-greedy policy of the trained agent to generate a dataset of 1,000 transitions. We use the same agent to generate 10 datasets under different seeds.

We then extract a 1-step unanchored coreset and a 10-step anchored coreset out of the full dataset. The two coresets have the same size of around 100 transitions, and they contain corresponding transitions $(s_{t+9},a_{t+9},r_{t+10}, \gamma, s_{t+10},a_{t+10})$ and $(s_t,a_t,g_{t, n} = \sum_{i=0}^{9} \gamma^{i} r_{t+i+1},\gamma^n, s_{t+10},a_{t+10})$ for $t\in\{0,10,20,\dots\}$, respectively. In other words, we control the setup so that the same bootstrap states and actions $(s_{t+10},a_{t+10})$ exist in the $1$-step coreset and the 10-step coreset. Note that $n$-step aggregation can be terminated early, if a terminal state is reached. For example, if we have $s_t, a_t, r_{t+1}, \gamma_{t+1}, s_{t+1}, a_{t+1}, r_{t+2}, \gamma_{t+2}, s_{t+2}, a_{t+2}$ and $\gamma_{t+2} = 0$ indicates termination, then we store $(s_t,a_t,g_{t, 2} = r_{t+1} + \gamma r_{t+2}, \gamma_n = 0, s_{t+2},a_{t+2})$. The update does not bootstrap off of $(s_{t+2},a_{t+2})$, because $\gamma_n = 0$, and $(s_{t+2},a_{t+2})$ are actually the pair at the start of the next episode. 

\begin{wrapfigure}[17]{r}{0.5\textwidth}
\vspace{-.6cm}
    \centering
    \includegraphics[width=\linewidth]{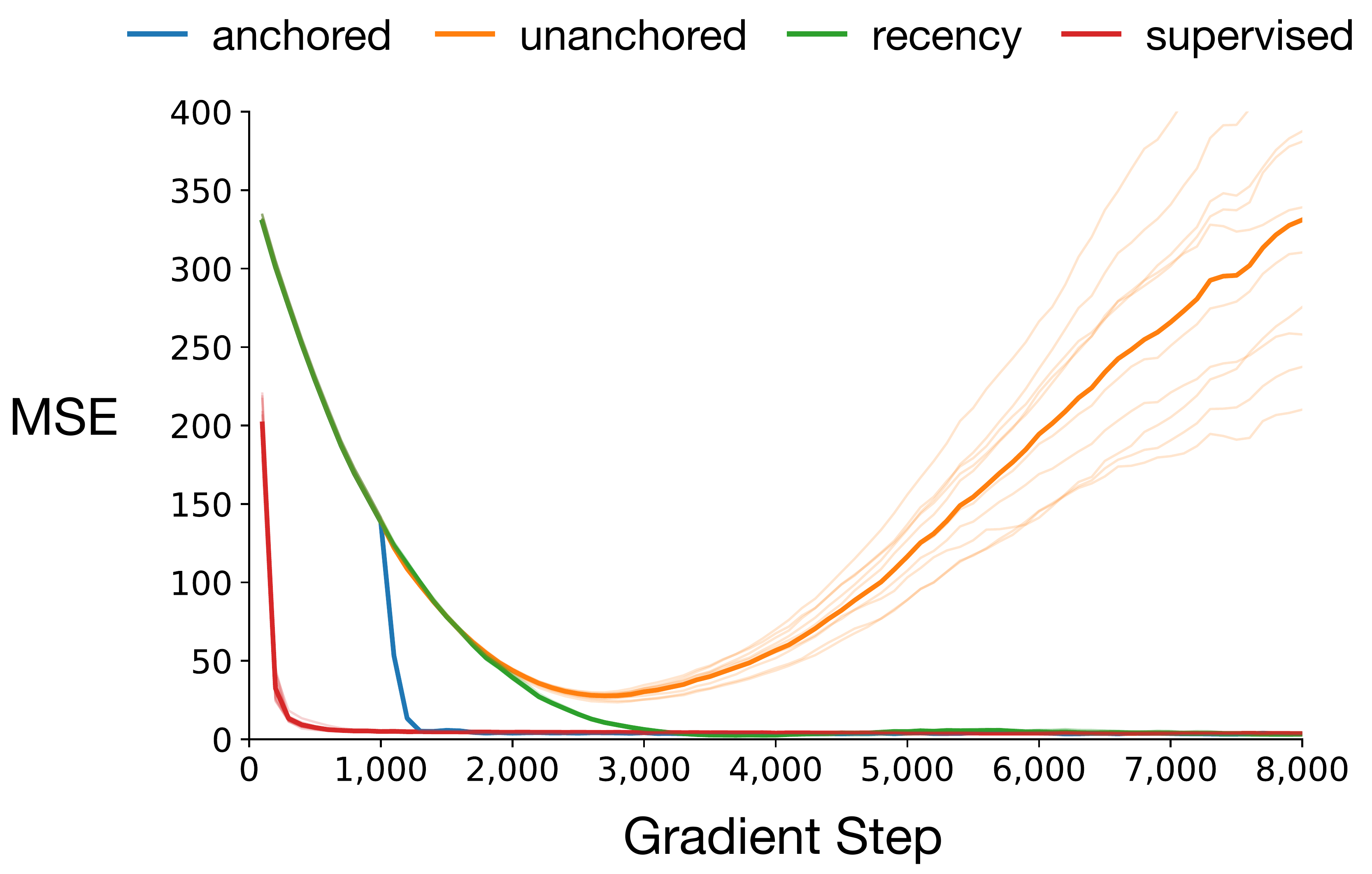}
    \caption{Mean square error between the bootstrapped target $Q(s_{t+10},a_{t+10})$ and sample return $g_{t+10}$ of anchored, unanchored, recency and supervised learners. Average over 100 seeds. Individual runs of 10 random seeds are shown in light thin lines to improve visibility.}
    \label{fig:pinball_prediction}
\end{wrapfigure}

For this experiment, we also store the Monte Carlo returns $G_{t+10}$ from the bootstrap state and action for each transition in the coreset, to compute the return error to evaluate the accuracy of the learned action-values for the bootstrap state and action. 
Specifically, we measure the deviation of the bootstrapped target $Q(s_{t+10},a_{t+10})$ to the sample return $G_{t+10}$ using the mean squared error. 

We have four learners: anchored, unanchored, recency, and supervised. The anchored, unanchored, and recency learners have the same architecture with online and target networks same as the agent that generates the dataset. They are trained with TD updates on randomly sampled mini-batches. The anchored and unanchored learners have a burn-in phase of 1,000 gradient updates, where they are trained on the full dataset. Then the anchored learner switches to the 10-step coreset, and the unanchored learner to the 1-step coreset. The recency learner is a baseline and is trained on the full dataset throughout. The other baseline is the supervised learner that directly learns on the sample returns $G_t$. All learners are trained with a total of 8,000 gradient updates.

As shown in Figure \ref{fig:pinball_prediction}, the anchored and supervised learners are able to learn the correct action-value of the bootstrapped targets. The recency learner, despite learning slowly, also finds the correct values eventually. In contrast, the unanchored learner diverges.

\subsection{Expectile Sarsa: reducing the pessimistic bias from the $n$-step targets in control}
\label{sec-expectile-sarsa}

In the prediction experiment above, there was no bias from using an $n$-step target, because all the data is generated on-policy. However, in control, the buffer is composed of transitions from older policies. The $n$-step target is biased, and likely pessimistic because the older policies were likely more suboptimal. Consider an example where state $s_{t+n}$ has high value and we have stored transition $(s_{t},a_{t},\sum_{i=0}^{n-1} \gamma^{i} r_{t+i+1},s_{t+n},a_{t+n})$ collected under older policies. The current policy may be able to get higher rewards when going from $s_{t}$ to $s_{t+n}$, or may be able to get there in fewer steps. 

This bias is potentially one of the reasons larger $n$ are not used in DQN algorithms, when $n$-step targets are used at all. Common choices are small values like $n = 3$ or 5 \citep{tang2023va,hessel2018rainbow}. This bias could be reduced by using importance sampling, but this is generally avoided due to high variance \citep{munos2016safe}.

We instead propose using an expectile loss to mitigate this pessimistic bias. While standard squared-error losses (like in DQN) estimate the expected value of a target, an expectile loss with parameter $\tau \in (0,1)$ learns a generalized mean called the expectile $u$. Specifically, $u$ is the point where the expected downward deviations (weighted by $1 - \tau$) balance the expected upward deviations (weighted by $\tau$). More precisely, the expectile $u$ for a random variable $X$ satisfies $(1 - \tau)\mathbb{E}_{x < u}[u - x] = \tau\mathbb{E}_{x > u}[x - u]$. As $\tau$ approaches $1$, $u$ shifts toward the maximum value of the distribution because the downward deviations must be heavily discounted to equal the remaining upward expectations. At $\tau = 0.5$, the expectile is the mean. To estimate $u$, we employ an asymmetric loss 
that separately weight positive and negative deviations $\delta = x - u$:
\begin{equation*}
    \ell_\tau(u) \doteq \tau \mathbb{I}(\delta \ge 0)\delta^2 + (1 - \tau) \mathbb{I}(\delta < 0)\delta^2.
\end{equation*}
The expectile for a random variables is defined as $e_{\tau}(Z) \doteq \arg\min_\expectile \mathbb{E}[\ell_\tau(Z - \expectile)]$.
When $\tau = 0.5$, this formulation simplifies to the standard mean-squared error. This asymmetric property is leveraged by Implicit Q-Learning (IQL)~\citep{kostrikovoffline}, an offline RL algorithm that uses expectiles to maximize over 1-step transitions without querying actions outside the dataset, effectively avoiding out-of-distribution issues.

We use the expectile loss, from $s_t, a_t$, on targets $G_t^{(n)}(Q) \doteq \sum_{i=0}^{n-1} \gamma^{i} r_{t+i+1} + \gamma^n Q(s_{t+n},a_{t+n})$. 
We use $Q(s_{t+n},a_{t+n})$, instead of the typical $\max_{\tilde{a}} Q(s_{t+n},\tilde{a})$ in control, to ensure we have anchored bootstrap targets. We need to use the $a_{t+n}$ actually taken by the agent in the dataset because then we can ensure we update $Q(s_{t+1},a_{t+1})$ with its own $n$-step target. If we used $\max_{\tilde{a}} Q(s_{t+n},\tilde{a})$, and $\tilde{a} \neq a_{t+n}$, then we would be querying $Q(s_{t+n},\tilde{a})$ for an anchored state but {\bf unanchored action}. We empirically investigate the utility of having both an anchored state and action, as opposed to only an anchored state in Section \ref{sec:anchor-control}. Because we use the actions from the policy that generated the data, we call this approach \emph{$n$-step Expectile Sarsa}. 

\textbf{Convergence under n-step expectile Bellman operators:} We can formalize the action-values that are being approximated by $n$-step Expectile Sarsa, and show that we can get sound Bellman updates with these action-values. 
For any policy $\pi$, we define the \emph{n-step expectile action-values} 
as
\begin{equation}
    Q^{(n)}_{\tau, \pi}(s,a) = e_\tau(G^{(n)}_t(Q^{(n)}_{\tau, \pi}))\label{eq_nstep_values}
\end{equation}
which is the expectile across all n-step returns under $\pi$ starting from $s,a$. A key difference from standard n-step boostrapping approaches in RL (including the standard 1-step bootstrap) is that we lose linearity due to the expectile. There is still a recursion, but it is now non-linear. Fortunately, we can still define an n-step expectile Bellman operator and prove convergence. We formalize this statement and prove the result in \textbf{Theorem \ref{thm_n_fixed}, Appendix \ref{app_bellman}}. This result also guarantees that a solution to the recursion in Equation \ref{eq_nstep_values} exists, and so $Q^{(n)}_{\tau, \pi}$ is well-defined. 

There has been some work formalizing Bellman operators under expectiles. In the robust setting, with $\tau < 0.5$ and 1-step bootstrapping, convergence with an expectile Bellman operator has been shown \citep{clavier2024bootstrapping}. We provide a more general result for $n \ge 1$; to the best of our knowledge, $n > 1$ with expectile Bellman operators has not previously been analyzed. Risk-sensitive algorithms have been proposed that again look at 1-step updates \citep{shen2014risk,luo2026actor}. Multi-step distributional reinforcement learning has been analyzed, but this involved considering the whole distribution \citep{tang2022thenature}.

\textbf{Ordering over $n$-step expectile values:} Another key difference from the standard expected return setting is that we get different action-values for different $n$. In the tabular setting, with expected return, different choices of $n$ simply give a bias-variance trade-off during learning but do not impact the final convergence point. For n-step expectile action-values, the choice of $n$ reflects more or less conservative value estimates. Specifically, we show in \textbf{Theorem \ref{thm_ordering}, Appendix \ref{app_ordered}} that for $\tau \ge 0.5$, for any multiplier $a \in \mathbb{N}$
$V^{(n)}_{\tau, \pi}(s) \ge V^{(a\cdot n)}_{\tau, \pi}(s)$. Intuitively, this makes sense, as taking the expectile over a longer n-step return has the chance for some cancellation in this sum as opposed to nesting the expectiles sooner. An important future direction is to better understand how to choose $n$, in terms of the quality of the policy we obtain when greedifying on $Q^{(n)}_{\tau, \pi}$.

\textbf{Convergence under policy iteration:} Ultimately, beyond ensuring soundness of our expectile updates for a fixed policy, we also want to do policy improvement. We show that we can get convergence to an optimal n-step expectile policy, under the standard policy iteration updates. As per above, the optimal policy could be different depending on the choice of $\tau$ and $n$, once again different from the expected return setting. This result also requires a stronger condition on the greedification that may not always be satisfied by soft-greedy policies like softmax or $\epsilon$-greedy policies.
We provide these conditions, prove convergence under policy iteration in \textbf{Theorem \ref{thm_main}} and provide more discussion in Appendix \ref{app_theory_pi}.

\textbf{Convergence to the optimal policy under deterministic dynamics:} We can finally also ask how these policies relate to the optimal policy. We can show that this procedure converges to the optimal policy, but under one key constraint: the environment dynamics are deterministic. The reason for this constraint is that the expectile loss can chase environment stochasticity rather than putting a higher weight on actions that result in higher expected return. For example, if the reward from $(s,a)$ has very high-variance, then $Q_{\tau, \pi}(s,a)$ may also be high because the expectile more heavily weights returns that have rewards in this higher tail.\footnote{IQL only used the expectile to maximize over actions in the dataset, specifically using $\ell_\tau(Q(s,a) - V(s))$ to learn a value function $V$ that reflected the values for a greedified policy (where the level of greedification/maximization is determined by $\tau$). The action-values are updated using one-step transitions that bootstrap off of $V$: $r + \gamma V(s')$. They therefore avoid this bias from maximizing over environment stochasticity.} If the environment is deterministic, then this issue does not arise, and instead the expectile only maximizes over the sequence of actions. In other words, the expectile puts higher weight on returns that resulted from better sequences of actions. %
The updates concentrate on good actions without erroneously chasing stochasticity. We show that policy iteration with the n-step expectile values does converge to the optimal policy in the deterministic setting, in Corollary \ref{cor_optimal} in Appendix \ref{app_theory_pi}.  

We expect empirically that $n$-step Expectile Sarsa will be robust to a small amount of environment stochasticity. With low environment stochasticity, using the expectile loss should largely be maximizing over action sequences, particularly in earlier learning when it is the dominant source of differences in returns. Bias due to stochasticity is also partially mitigated because the coreset buffer is gradually shifting. It can also be explicitly reduced by decaying $\tau$ to 0.5 to shift to using the standard squared error which estimates the expected return. In its current form, however, this algorithm would not be appropriate for environments with high stochasticity.

\subsection{The Complete Endpoint Replay Algorithm}

The final piece of the algorithm is to use two buffers: a small recency buffer for regular one-step updates on new data and the coreset buffer with $n$-step targets to reinforce values on older data. This two-buffer approach has been previously used in continual RL to balance fast tracking and preventing forgetting~\citep{yang2024augmenting,anand2026permanent}. The recency buffer $\recencybuffer$ simply stores the most recent $j$ points and the coreset buffer $\corebuffer$ summarizes the experience $j+1$ steps ago and older. By default, we pick a 90-10 rule: when trying to mimic a large recency buffer of size one million with approximately 10\% of the storage, we use $j = 10k$ for $\recencybuffer$ and $90k$ for the coreset $\corebuffer$.

Our implementation technically maintains a third buffer, to convert the one-step transitions that age out of the small recency buffer into $n$-step updates for the coreset. This tiny \emph{lag} buffer stores up to $n$ transitions (e.g., $n=10$), moved from the recency buffer to the end of the lag buffer. Once the lag buffer fills with $n$ transitions, the $n$-step return is computed, and the transition $(s_t, a_t, \sum_{i=0}^{n-1} \gamma^i r_{t+i+1} , s_{t+n}, a_{t+n})$ is added to the coreset $\corebuffer$.
One small nuance is if termination occurs before $n$ steps, say in $k < n$ steps. In this case, the shorter $k$-step return is computed and the $k$-step transition added to the coreset buffer. The lag buffer would simply reset after $k$ steps.   

On each step, a mixed mini-batch of samples from $\recencybuffer$ and $\corebuffer$ are used, with the expectile Sarsa update used for samples from $\corebuffer$ and standard DDQN updates used on samples from $\recencybuffer$. One question is how many samples to use from each.
The small recency buffer allows for updating on new experience, and the coreset acts to reinforce updates across the space. We found it more effective to sample more from the small recency, to learn more from these new experience, allowing for the samples from the coreset to play more of a regularizing effect. In addition, the expected number of updates per sample is highly sensitive in deep RL \citep{fedus2020revisiting}; by sampling less from coreset we bring expected update per coreset sample down to similar values had they been in a large recency buffer. Specifically, in our experiments we found a relatively high ratio of 7:1 to be effective. 
For example, for a mini-batch of size 32, we use 28 samples from $\recencybuffer$ and 4 from $\corebuffer$.

We summarize this in Algorithm \ref{alg:endpoint_replay}, in Appendix \ref{app_algspec}. We call the algorithm Endpoint Replay, because a key design of the algorithm is to anchor endpoints in the coreset.

\section{The impact of unanchored states in control}
\label{sec:anchor-control}
This section investigates the impact of unanchored targets in the control setting. Earlier we showed unanchored bootstrap values can get worse over time when learning from isolated samples in the prediction setting. We hypothesize that the same effect also occurs in the control setting: unanchored bootstrap values degrade over time. Measuring and analyzing bootstrap value errors in the control setting is more difficult than in prediction because the policy continually changes meaning the value estimates and the ground truth values change over time, and even the return is no longer an unbiased estimate of the value.

To navigate the difficulty inherent in measuring value errors in control we designed the following experiment. First a DDQN agent with a large recency buffer is trained until a good policy is learned. This ensures the agent's policy and value estimates are reasonably accurate. Then the agent's buffer is replaced with an Endpoint replay buffer filled using with the agent's interaction history. The agent then continues to interact with the environment, adding new data to both its recency and coreset buffers, and learn from buffer mini-batches. The goal is to check whether (1) an agent can maintain good performance (in terms of return) after switching to Endpoint replay compared with an unanchored coreset baseline and (2) the values of the boostrap states more accurate compared with an unanchored coreset baseline.

To measure bootstrap value errors, we compare the agent's current target network estimate at bootstrap states in the buffer to the returns observed at the completion of episodes of those bootstrap states. The return measured at the end of a previous episode is not an unbiased estimate of the current value function, it is rather a rough approximation that we control by starting the experiment with a good policy and keeping the buffer size reasonable. 
We compute the squared error for all bootstrap states in the buffer and report the average error.

We use the Pinball domain~\citep{konidaris2009skill,lo2024goal} and start by training a tuned DDQN agent with a recency buffer of size 10k for 50k time steps until it reaches good but not optimal performance. Then we swap this recency buffer with a replay buffer with recency buffer of size 100 and coreset buffer of size 900. This Endpoint replay agent performs 10-step sub-sampling. The agent continues learning for another 50k steps over which we measure the undiscounted return and the root mean square of coreset buffer's bootstrap target value errors as discussed above. We compare Endpoint replay with two unanchored baselines using 1-step DDQN updates. The first coreset, which we call Interval, sub-samples transitions every 10 steps, similar to the unanchored algorithm in Figure \ref{fig:pinball_prediction}. The second coreset is based on reservoir sampling \citep{isele2018selective} and maintains a uniform sample of the history of the agent within the coreset buffer. Both of these baselines do not anchor states or actions. Each algorithm is run for 100 independent trials. See the supplementary materials for environment details and list of hyperparameters.

\begin{figure}[htb]
    \centering
    \includegraphics[width=0.9\textwidth]{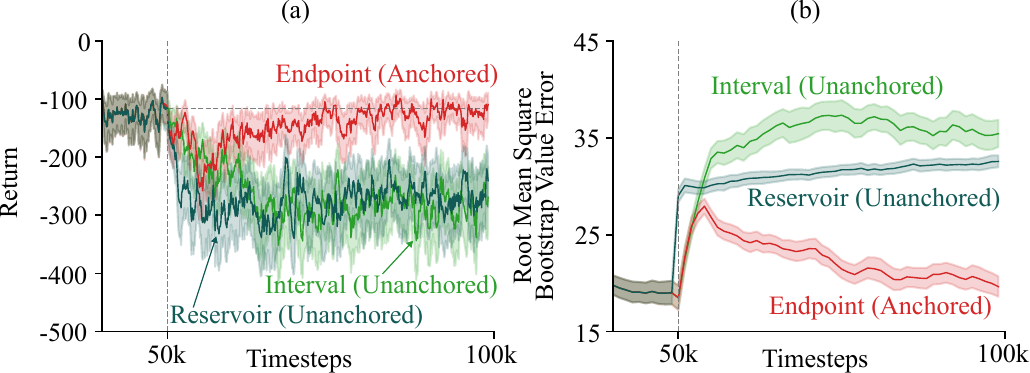}
    \centering
    {\phantomsubcaption\label{fig:anchor-control-1000-a}}%
    {\phantomsubcaption\label{fig:anchor-control-1000-b}}%
    \caption{Endpoint replay can maintain better performance (left) and has lower bootstrap target value error (right) compared with both unanchored coreset baselines. The vertical dashed line indicates the moment we switch from a large recency buffer to a smaller coreset buffer. The horizontal dashed line marks the performance of the agent before the switch. Results are averaged over 100 seeds, shaded regions are 95\% bootstrap confidence intervals.}
    \label{fig:anchor-control-1000}
\end{figure}

Figure \ref{fig:anchor-control-1000} shows the outcome of the experiment. After switching to the coreset, all three agents experience a dip in performance. Endpoint replay recovers quickly, with the initial dip due to the sudden switch to n-step updates. The unanchored agents performance just degrades to a suboptimal policy as we see in Figure \ref{fig:anchor-control-1000-a}. This trend in performance is correlated with the measured bootstrap target value error of transitions in the buffer as shown in Figure \ref{fig:anchor-control-1000-b}. We see an increase in the value error following the switch. After some time we see Endpoint's bootstrap value errors reduce while errors continue to increase in the unanchored coresets. %

The results presented in this section as well as those in the prediction setting presented earlier, show that isolated, unanchored transitions in a coreset buffer can lead to poor bootstrap value estimates and harm performance. We showed how Endpoint replay mitigates this issue by anchoring states with chained n-step targets and anchoring actions with expectile Sarsa updates, reducing bootstrap value error. In the next section we will demonstrate that the differences we see here in this controlled experiment translate into performance differences. %

\section{Benchmark Experiments for Endpoint replay}
\label{sec:control-perf}
In this section we benchmark Endpoint replay in two domains against several baselines. We use the Pinball environment and a subset of games in the Arcade Learning Environment \citep{bellemare2013arcade}. We also present ablation studies of Endpoint replay, demonstrating that anchoring
states and actions is crucial for strong performance, and that expectile updates appear to reduce $n$-step bias.

First, we compare the performance of Endpoint replay against several baselines in Pinball. We test a large 10k recency buffer, a small recency buffer, a small recency buffer with 10-step updates. The inclusion of the small recency buffer with 10-step helps distinguish if the benefits of Endpoint replay are completely explained by the use of n-step updates. We include a DDQN variant of ~\citeauthor{lan2022memory}'s MeDQN, which must sample random state-action pairs in order to compute the regularization to the target network values. This approach risks generating invalid or never visited state-action pairs, but is unavoidable without saving a large buffer which would defeat the purpose. We also include a coreset method based on reservoir sampling~\citep{isele2018selective} using DDQN updates. Reservoir sampling ensures transitions added to the coreset are sampled uniformly over the agent's history, though they are still isolated and unanchored. We consider two settings where the Endpoint replay buffers and small buffers are 10 and 20 times smaller than the big buffer (1000 and 500 transitions).

There are several important details to note.
The batchsize of all agents is 32, Endpoint replay chooses 28 samples from the recency buffer and 4 samples from the coreset. All agents use a two layer neural network with 32 hidden units and ReLU activations updated with the Adam optimizer. The 1k Endpoint agent uses a 100 transition recency buffer and a 900 transition coreset. The 500 Endpoint agent uses 100 and 400 transitions for the recency and coreset buffers.
The learning rate is tuned for the large recency baseline and the same value is used across all agents. The experiment was run for 100k time steps and averaged over 100 independent trials. MeDDQN used a buffer size of 1000 and 500. Prior work~\citep{lan2022memory} suggested a batchsize of 32 could work but we found that performed quite badly. MeDDQN has a hyperparameter $\lambda$ which weights the loss. We swept three different values of $\lambda \in \{1, 2, 4\}$ with 30 seeds each, then ran the best performing value ($\lambda=2$) for another 100 seeds following the two-stage tuning approach~\citep{patterson2024empirical}. We report undiscounted return over time as performance metric. See the supplementary materials for the environment details, and list of hyperparameters used.%

\begin{figure}[ht]
    \centering
    \includegraphics[width=0.95\textwidth]{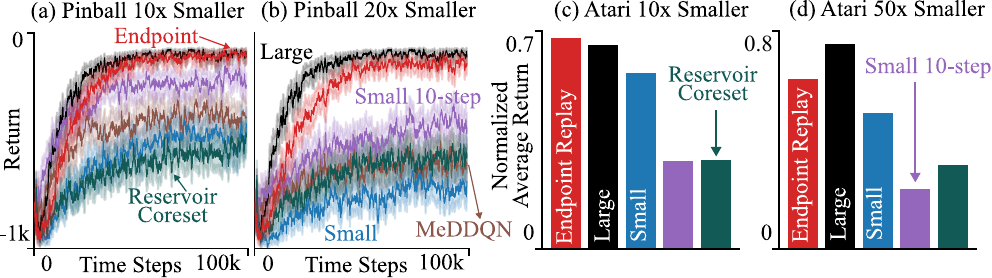}
    \centering
    {\phantomsubcaption\label{fig:control-perf-a}}%
    {\phantomsubcaption\label{fig:control-perf-b}}%
    {\phantomsubcaption\label{fig:control-perf-c}}%
    {\phantomsubcaption\label{fig:control-perf-d}}%
    \caption{Endpoint replay performs comparably to a much larger recency buffer while outperforms an equal size recency buffer. In Pinball we report mean performance over 100 seeds and shaded regions are 95\% bootstrap confidence intervals. Atari results are aggregated scores across 12 games and 10 seeds, min-max normalized.}
    \label{fig:control-perf}
    \vspace{-0.5cm}
\end{figure}

Figures \ref{fig:control-perf-a} and \ref{fig:control-perf-b} summarize the results in Pinball. In both cases (10x and 20x smaller) Endpoint replay outperforms the small but an equal size recency buffer and approaches the performance of the tuned, much larger, recency buffer. The small recency baseline with 10-step targets does perform better than its 1-step counterpart but still under-performs compared to Endpoint replay. Interestingly, in the 20x smaller case, the performance of the small buffer and small 10-step buffer drops, whereas Endpoint's performance hardly changes. MeDDQN performed  poorly, likely because generating random states in Pinball is problematic due to obstacles and the fact that many combinations of velocity and position are not possible. Similarly reservoir sampling also performed poorly in Pinball suggesting unanchored updates can degrade performance even when the transitions are sampled over the agent's entire lifetime.

Next we investigate performance of Endpoint replay in larger pixel-based environments. We select 12 games from the Arcade Learning Environment \citep{bellemare2013arcade}, including the Atari-5 games \citep{aitchison2023atari}, and 7 games selected to highlight a diverse range of problems. See the supplementary materials for details about the selected games. Much like the Pinball experiments, we compare Endpoint replay with a recency buffer that is 10 times larger, consisting of one million samples, as is typical in Atari. Here the small baselines and Endpoint use 100k transitions.

We keep the relative sizes of buffers the same as in the Pinball experiment.
In the 100k setting (10x smaller), we allocate 10k to the Endpoint replay's recency buffer and 90k to its coreset. In the 20k setting (50x smaller) we allocate 10k to the Endpoint replay's recency buffer and 10k to its coreset. The coreset sub-samples every 10 steps while accumulating 10-step targets. The batchsize is 32 and like Pinball, we sample 28 from the recency buffer and 4 from the coreset. All the other agent hyperparameters and network structure follow the Dopamine baseline which closely follows the original DQN and DDQN implementations \citep{mnih2013atari,van2016deep} except using the Adam optimizer which is shown to outperform RMSprop \citep{hessel2018rainbow}. We run the experiment for 50 million frames (to save computation but still allow plenty of time for learning, as in prior work~\citep{lan2022memory}) and report undiscounted return observed online during training. Performance is normalized by the best and worst performing seed in each game across algorithm variations. Each variation is run for 10 independent trials. We do not include MeDDQN to reduce computational costs as it performed so poorly in Pinball but do include reservoir sampling as another baseline with unanchored updates. See the supplementary materials for details about selected games, environment settings, and list of hyperparameters.

Figure \ref{fig:control-perf-c} summarizes the results when the Endpoint buffer is 10x smaller. Endpoint replay outperforms equal size recency buffer (small)\footnote{In the 100k setting Endpoint outperforms Small significantly according to the sign test: better in 82 / 120 paired seeds.} and outperforms the small recency buffer with 10-step targets. In this particular setting, Endpoint replay performs even better than the large recency baseline. Looking at the individual game performance in supplement Figure ~\ref{fig:atari_100k}, we see Endpoint ties the large buffer in 1 game, outperforms in 6 games, and does slightly worse in 5. When Endpoint wins, it does so by a slightly larger margin because there are some games where a small buffer outperforms a large buffer, which has been observed before ~\citep{fedus2020revisiting,DBLP:journals/corr/abs-1712-01275}. Reservoir coreset performs poorly and similar to small recency buffer with 10-step targets. In the 50x smaller setting, shown in Figure \ref{fig:control-perf-d} we see Endpoint replay outperforming both 1-step\footnote{In the 20k setting Endpoint outperforms Small significantly according to the sign test: better in 80 / 120 paired seeds.} and 10-step small recency baselines but this time does not quite reach the performance of large baseline. Even with a 50x smaller buffer Endpoint replay performs at least as well as the large buffer in 5 of 12 games as show in supplement Figure \ref{fig:atari_20k}. Reservoir sampling under performs Endpoint replay and both 1-step recency baselines in this setting as well. Per game learning curves are provided in supplement Figures \ref{fig:atari_100k_curves} and \ref{fig:atari_20k_curves}.

As in Pinball, we again see Endpoint replay outperforms the same sized small buffer outfit with 10-step updates.
This suggests the different algorithm components of Endpoint replay work together to achieve good performance. %

We preformed an ablation study to investigate importance of each of the major algorithmic choices in Endpoint replay. The experiment setup is identical to those presented earlier in this section. Specifically we consider the following ablations (1) squared versus expectile loss, (2) anchored versus unanchored actions (Sarsa vs DDQN updates), and (3) anchored versus unanchored states and actions (1-step vs chained 10-step). The 1-step unanchored variant uses squared loss and DDQN updates identical to the Interval coreset used in Sections \ref{sec:anchor-pred} and \ref{sec:anchor-control}.

\begin{figure}[bt]
    \centering
    \includegraphics[width=0.95\textwidth]{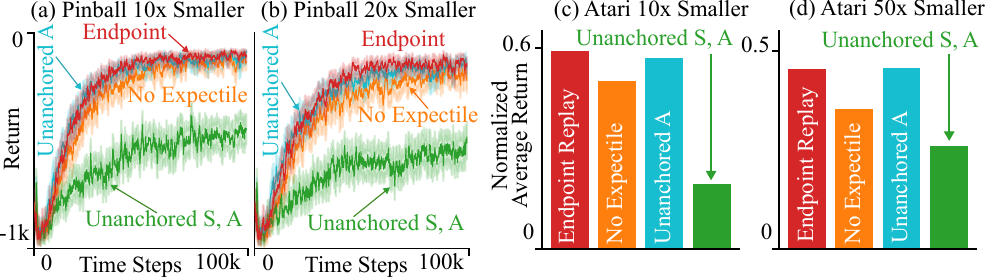}
    
    {\phantomsubcaption\label{fig:control-ablations-a}}%
    {\phantomsubcaption\label{fig:control-ablations-b}}%
    {\phantomsubcaption\label{fig:control-ablations-c}}%
    {\phantomsubcaption\label{fig:control-ablations-d}}%
    \caption{Ablation study of Endpoint replay in Pinball and Atari. In Pinball we report mean performance over 100 seeds. The shaded regions are 95\% bootstrap confidence intervals. Atari results are aggregated scores across 10 seeds in each of 12 games, min-max normalized.} 
    \label{fig:control-ablations} 
\end{figure}

The results are presented in Figure \ref{fig:control-ablations}. In Pinball (Figures \ref{fig:control-ablations-a} and \ref{fig:control-ablations-b}) we see Endpoint replay performs best. Anchoring states and actions appears to be crucial while anchoring actions is not as important. The expectile loss does matter and this effect is more pronounced in 20x smaller setting. The results in Atari games (Figures \ref{fig:control-ablations-c} and \ref{fig:control-ablations-d}) largely agrees with Pinball results. Namely anchoring state and expectile updates matter most.\footnote{Endpoint outperforms No Expectile significantly according to the sign test: better in 79 / 120 paired seeds in the 100k setting and better in 93 / 120 paired seeds in the 20k setting.} Per game results can be found in supplement Figures \ref{fig:ablations_atari_100k} and \ref{fig:ablations_atari_20k} as well as learning curves in Figures \ref{fig:ablations_atari_100k_curves} and \ref{fig:ablations_atari_20k_curves}.

The ablation study presented above shows that the main design components of Endpoint replay, namely: anchoring the bootstrap state with chained n-step targets and correcting pessimistic bias of on-policy n-step SARSA updates with expectile loss contribute to the final performance. Although, anchoring the action mattered less in terms of performance, we view this as a positive result because action anchoring is more correct than not and it does not appear to hurt performance.

\section{Conclusion}
\label{sec:conclusion}
In this paper we introduced Endpoint replay, a simple approach designed to compress a large recency buffer into a much smaller coreset and recency buffer with the goal to maintain comparable performance. We found we could reduce the size of the buffer by 10x or even 50x, and indeed maintain comparable performance to using a large recency buffer across twelve Atari environments. To motivate the algorithm, we identified a key issue with extending the typical approach to creating coresets to DRL: they produce isolated transitions that have unanchored bootstrap targets. These unanchored targets can cause the action-values to become highly inaccurate, as we demonstrated in carefully controlled experiments in both prediction and control. We resolved this using $n$-step returns, to ensure each bootstrap target remains anchored in the coreset. To mitigate the pessimistic bias from  using $n$-step returns, we introduced an $n$-step Expectile Sarsa and showed empirically that it improved performance over using a standard $n$-step update. 

\subsubsection*{Broader Impact Statement}
\label{sec:impact}
Our work focuses on making deep RL methods use less resources while maintaining performance. Though this has a clear benefit in terms of reducing the environmental impact of training costs and democratizing research, one could simply use our advances to run even more experiments or larger agents muting the benefits. Our research is fundamentally algorithmic and applied to simulation problems that bare little resemblance to reality. Nevertheless we urge users of our research to consider the impacts of applying our work to real-world problems and in other scopes.

\bibliography{main}
\bibliographystyle{rlj}

\beginSupplementaryMaterials

\section{Theory for expectile operators and action-values}\label{app_proof}

Throughout this section we will use the following expectile operator. We introduce the definition of the expectile with the influence function because it is convenient for the proofs. 
\begin{definition}
Define the \emph{expectile operator} for the random variable of returns $Z$ as 
\begin{align*}
    e_{\tau}(Z) \doteq \arg\min_\expectile \mathbb{E}[\ell_\tau(Z - \expectile)] \quad\quad \text{ for } \ell_\tau(\delta) \doteq \tau \mathbb{I}(\delta \ge 0)\delta^2 + (1 - \tau) \mathbb{I}(\delta < 0)\delta^2
\end{align*} 
This can be equivalently defined using the influence function 
\begin{equation*}
    \psi_\tau(x) \doteq 2\vert{}\tau - \mathbb{I}(x < 0)\vert{}x
\end{equation*} 
where the expectile value satisfies
\begin{align*}
     \mathbb{E}[\psi_\tau(Z - e_{\tau}(Z))] = 0
\end{align*} 
\end{definition}
There are several well-known properties of expectiles.
\begin{enumerate}
\item \textbf{[Uniqueness]} For any bounded $X$ ($|X|\le B$), there is a unique $m\in[-B,B]$ with $H_X(m)\doteq E[\psi_\tau(X-m)]=0$.
\item \textbf{[Translation invariance]} $e_\tau(X+c)=e_\tau(X)+c$ for any constant $c$. 
\item \textbf{[Monotonicity]} If $X\le Y$ a.s., then $e_\tau(X)\le e_\tau(Y)$.
\item \textbf{[Subadditivity]} For $\tau \ge 0.5$, $e_\tau(X + Y) \le e_\tau(X) + e_\tau(Y)$
\end{enumerate}

\subsection{n-step expectile values and Bellman operators}\label{app_bellman}

We can define the $n$-step expectile values and associated $n$-step expectile Bellman operator. 
Assume rewards are a.s. bounded, $|R|\le R_{\max}$. Let $\mathcal V=\{v:\mathcal S\to\mathbb R : \|v\|_\infty<\infty\}$ with $\|v\|_\infty=\sup_s|v(s)|$.
For $v\in\mathcal V$, $s\in\mathcal S$, $n \ge 1$, define the n-step bootstrapped return
\begin{equation*}
    G_s^{(n)}(v) \;\doteq\; \sum_{k=0}^{n-1}\gamma^k R_{k+1} \;+\;\gamma^n v(S_n)
\end{equation*}
where $S_0=s$ and $A_k\sim\pi(\cdot\mid S_k)$. Since $|G_s^{(n)}(v)|\le\frac{1-\gamma^n}{1-\gamma}R_{\max}+\gamma^n\|v\|_\infty$, it's a.s. bounded. $G_{s,a}^{(n)}(q)$ is defined analogously, assuming $(s,a)$ is given and we bootstrap with $q$. 
The n-step expectile Bellman operator is defined as
\begin{equation}
    (T_{\tau,\pi}^{(n)}v)(s) \doteq e_\tau\big(G_s^{(n)}(v)\big).
\end{equation}
We prove this has a unique fixed point, in Theorem \ref{thm_n_fixed}, giving the corresponding n-step expectile values 
\begin{align}
V^{(n)}_{\tau, \pi}(s) &\doteq  e_{\tau}(G_s^{(n)}(V^{(n)}_{\tau, \pi}))\\ 
Q^{(n)}_{\tau, \pi}(s,a) &\doteq  e_{\tau}(G_{s,a}^{(n)}(Q^{(n)}_{\tau, \pi}))
\end{align}
Note that for $n$ large enough, the $n$-step returns become Monte Carlo returns, and we define the limit of $n$ to be the expectile action-values and omit the $n$ rather than writing $n = \infty$:
\begin{align}
V_{\tau, \pi}(s) &\doteq  e_{\tau}(G_s(V_{\tau, \pi}))\\ 
Q_{\tau, \pi}(s,a) &\doteq  e_{\tau}(G_{s,a}(Q_{\tau, \pi}))
\end{align}

\begin{theorem} \label{thm_n_fixed}
For any $\tau\in(0,1)$, $n\ge1$: 
\begin{equation}
    \|T_{\tau,\pi}^{(n)}v_1-T_{\tau,\pi}^{(n)}v_2\|_\infty \le \gamma^n\|v_1-v_2\|_\infty
\end{equation} 
for all $v_1,v_2\in\mathcal V$. Therefore this operator has a unique fixed point $V_{\tau,\pi}^{(n)}$, reached from any starting $v_0 \in \mathcal V$. 
\end{theorem}
\begin{proof}
Fix $v_1,v_2\in\mathcal V$, let $\Delta=\|v_1-v_2\|_\infty$, and fix $s\in\mathcal S$. Assume we see an n-step trajectory  $(S_0=s,A_0,R_1,S_1,\dots,A_{n-1},R_{n},S_n)$. Notice that 
$$v_1(S_n)\le v_2(S_n)+\Delta$$
because $\Delta$ bounds the gap at every state, including at $S_n$. Because they have the same n-step $\sum_{k=0}^{n-1}\gamma^k R_{k+1}$, we get
\begin{equation*}
    G_s^{(n)}(v_1)=\sum_{k=0}^{n-1}\gamma^k R_{k+1} +\gamma^n v_1(S_n) \;\le\; \sum_{k=0}^{n-1}\gamma^k R_{k+1} +\gamma^n\big(v_2(S_n)+\Delta\big) = G_s^{(n)}(v_2)+\gamma^n\Delta
    .
\end{equation*}
Using monotonicity and translation invariance, we get
\begin{align*}
    e_\tau\big(G_s^{(n)}(v_1)\big) &\le e_\tau\big(G_s^{(n)}(v_2)+\gamma^n\Delta\big) && \triangleright \text{ monotonicity of $e_\tau$}\\
    &= e_\tau\big(G_s^{(n)}(v_2)\big)+\gamma^n\Delta && \triangleright \text{ translation invariance of $e_\tau$}
\end{align*}
This holds for every $s\in\mathcal S$. 

Repeating the identical argument with $v_1,v_2$ swapped gives $(T_{\tau,\pi}^{(n)}v_2)(s)-(T_{\tau,\pi}^{(n)}v_1)(s)\le\gamma^n\Delta$ for every $s$. Putting this all together
\begin{equation*}
    \big|(T_{\tau,\pi}^{(n)}v_1)(s)-(T_{\tau,\pi}^{(n)}v_2)(s)\big|\le\gamma^n\Delta \quad \forall s \;\Rightarrow\; \|T_{\tau,\pi}^{(n)}v_1-T_{\tau,\pi}^{(n)}v_2\|_\infty\le\gamma^n\Delta
\end{equation*}
Because $T_{\tau,\pi}^{(n)}$ is a $\gamma^n$-contraction on the complete space $(\mathcal V,\|\cdot\|_\infty)$, by Banach's fixed-point theorem we know it has a unique fixed point $V_{\tau,\pi}^{(n)}$, reached from any starting $v_0\in \mathcal V$.
\end{proof}
This can all be analagously shown for the n-step expectile action-values, by analyzing the n-step Bellman operator for action-values.

\subsection{Ordered relationship between n-step values}\label{app_ordered}

It is important to note that unlike n-step values for the expected return (when $\tau = 0.5)$, these n-step values are not all equal to each other, even in the tabular setting. We can show they all converge under Bellman iteration (as per Theorem \ref{thm_n_fixed}), but they will likely converge to different values depending on the choice of $n$. 
For larger $n$, the values are smaller, because we are taking the expectile over longer trajectories where the sums can cancel out some of the stochasticity. Bootstrapping earlier, on expectile values, results in larger values. Nesting the expectiles takes expectiles over another expectile, which is larger, rather than just one expectile over a random trajectory. We show that there is a strict ordering in Theorem \ref{thm_ordering}. 

\begin{lemma}[Subtower Property]\label{lem_subtower}
For $\tau \ge 0.5$, and random variables $X, Y$
$e_\tau(Y) \le e_\tau(e_\tau(Y \mid X))$
\end{lemma}
\begin{proof}
Let the random variable $U = e_\tau(Y \mid X)$.
By definition of the expectile, we have 
\begin{equation*}
\mathbb{E}[\psi_\tau(Y - U) \mid X] = \mathbb{E}[\psi_\tau(Y - e_\tau(Y \mid X)) \mid X] = 0
\end{equation*}
because $e_\tau(Y \mid X)$ is the expectile of the RV $Y \mid X$. 
We can show $\mathbb{E}[\psi_\tau(Y - U)] = 0$, implying the expectile value for the random variable $Y-U$ is 0, i.e., $e_\tau(Y - U) = 0$ because $\mathbb{E}[\psi_\tau(Y - U)] = \mathbb{E}[\psi_\tau(Y - U - m)]$ with $m = 0$. 
\begin{equation*}
\mathbb{E}[\psi_\tau(Y - U)] = \mathbb{E}[\mathbb{E}[\psi_\tau(Y - U) \mid X]] = \mathbb{E}[0] = 0
\end{equation*}
Using subadditivity of the expectile for $\tau \ge 0.5$, we get
\begin{equation*} 
e_\tau(Y) = e_\tau(U + (Y - U)) \le e_\tau(U) + e_\tau(Y - U) = e_\tau(U) = e_\tau(e_\tau(Y \mid X))
\end{equation*}
completing the proof.
\end{proof}

\begin{lemma}[Expectile Decomposition Lemma]\label{lemma_decomposition} For any integers $a, b \ge 1$ and any bounded value function $u \in \mathcal{V}$:
\begin{equation}
T_{\tau, \pi}^{(a+b)} u \le T_{\tau, \pi}^{(a)}\big(T_{\tau, \pi}^{(b)} u\big)
\end{equation}
\end{lemma}
\begin{proof}
For simplicity of notation, we drop the subscripts and just write $T^a$ for the operators. 

By definition of the $(a+b)$-step operator:
\begin{equation*}
(T^{(a+b)} u)(s) = e_\tau\left( \sum_{j=0}^{a+b-1} \gamma^j R_{j+1} + \gamma^{a+b} u(S_{a+b}) \;\middle\vert{}\; S_0 = s \right)
\end{equation*}
If we split the discounted sum at step $a$ we get:
\begin{equation*}
(T^{(a+b)} u)(s) = e_\tau\left( \sum_{j=0}^{a-1} \gamma^j R_{j+1} + \gamma^a \left[ \sum_{i=0}^{b-1} \gamma^i R_{a+i+1} + \gamma^b u(S_{a+b}) \right] \;\middle\vert{}\; S_0 = s \right)
\end{equation*}
Let $\mathcal{F}_a$ denote the filtration (history) up to time step $a$. The tower property of expectiles with $\tau \ge 0.5$ states that for any random variables $X, Y$ we have $e_\tau(Y) \le e_\tau\big(e_\tau(Y \mid X)\big)$. Using this, we know that conditioning on $\mathcal{F}_a$ inside the expectile cannot decrease the value,
\begin{align*}
(T^{(a+b)} u)(s) &\le e_\tau\left( e_\tau\left( \sum_{j=0}^{a-1} \gamma^j R_{j+1} + \gamma^a \left[ \sum_{i=0}^{b-1} \gamma^i R_{a+i+1} + \gamma^b u(S_{a+b}) \right] \;\middle\vert{}\; \mathcal{F}_a \right) \;\middle\vert{}\; S_0 = s \right)\\
&= e_\tau\left( \sum_{j=0}^{a-1} \gamma^j R_{j+1} + \gamma^a e_\tau\left( \sum_{i=0}^{b-1} \gamma^i R_{a+i+1} + \gamma^b u(S_{a+b}) \;\middle\vert{}\; \mathcal{F}_a \right) \;\middle\vert{}\; S_0 = s \right)
\end{align*}
where the second step follows from that fact that
$\sum_{j=0}^{a-1} \gamma^j R_{j+1}$ and $\gamma^a$ are not random give $\mathcal{F}_a$ and so can be pulled out of the expectile (the inner one in this case). 
By the Markov property of the MDP, the inner expectile depends only on the state $S_a$:
$$e_\tau\left( \sum_{i=0}^{b-1} \gamma^i R_{a+i+1} + \gamma^b u(S_{a+b}) \;\middle\vert{}\; S_a \right) = (T^{(b)} u)(S_a)$$
Substituting this back into the outer expectile yields:
$$(T^{(a+b)} u)(s) \le e_\tau\left( \sum_{j=0}^{a-1} \gamma^j R_{j+1} + \gamma^a (T^{(b)} u)(S_a) \;\middle\vert{}\; S_0 = s \right) = \big(T^{(a)}(T^{(b)} u)\big)(s)$$
\end{proof}

\begin{theorem}\label{thm_ordering}
For any $\tau\ge1/2$, for every $n\in \mathbb{N}$ and multiplier $a \in \mathbb{N}$  ($n, a \ge 1$): 
\begin{equation}
    V_{\tau,\pi}^{(n)}(s)\ge V_{\tau,\pi}^{(a\cdot n)}(s)
\end{equation} 
for every $s$. 
\end{theorem}
\begin{proof}
Pick any $n\in \mathbb{N}$ and any multiplier $a \in \mathbb{N}$. As in Lemma \ref{lemma_decomposition}, to simplify notation for this proof, we write $T^{(n)}$ instead of $T_{\tau, \pi}^{(n)}$, because it is clear here that we are taking about the expectile Bellman operators.

\textbf{Step 1:}
We first show that 
$$T^{(a \cdot m)} u \le T^{(a)}\big(T^{(a(m-1))} u\big) \le \big(T^{(a)}\big)^2\big(T^{(a(m-2))} u\big) \le \dots \le \big(T^{(a)}\big)^m u$$
We can show this using induction. 
For $m=1$, the statement is trivial: $T^{(a)} u \le (T^{(a)})^1 u$.

For $m=2$, we set $p = a$ and $q = a$ and apply the expectile decomposition lemma (Lemma \ref{lemma_decomposition}) directly to $u$:
\begin{equation*}
T^{(2a)} u = T^{(a+a)} u \le T^{(a)}\big(T^{(a)} u\big) = \big(T^{(a)}\big)^2 u
\end{equation*}

For $m=3$
\begin{equation*}
T^{(3a)} u = T^{(a+2a)} u \le T^{(a)}\big(T^{(2a)} u\big)
\end{equation*}
For $m=2$ case we already established that $T^{(2a)} u \le \big(T^{(a)}\big)^2 u$, and because the operator $T^{(a)}$ is monotone, is preserves order we can apply $T^{(a)}$ to both sides without changing the inequality:
\begin{equation*}
T^{(a)}\Big( T^{(2a)} u \Big) \le T^{(a)}\left( \big(T^{(a)}\big)^2 u \right) = \big(T^{(a)}\big)^3 u
\end{equation*}
giving
$T^{(3a)} u \le T^{(a)}\big(T^{(2a)} u\big) \le \big(T^{(a)}\big)^3 u$.

For the general inductive step $k \ge 1$, assume $T^{(ak)} u \le \big(T^{(a)}\big)^k u$.
To prove the step for $k+1$, we split the $(a(k+1))$-step horizon into $p = a$ and $q = ak$ using decomposition lemma
\begin{equation*}
T^{(a(k+1))} u = T^{(a + ak)} u \le T^{(a)}\big(T^{(ak)} u\big)
\end{equation*}
 and use $T^{(ak)} u \le (T^{(a)})^k u$ and the same monotone operator argument to get
\begin{equation*}
T^{(a(k+1))} u \le T^{(a)}\big(T^{(ak)} u\big) \le T^{(a)}\left( \big(T^{(a)}\big)^k u \right) = \big(T^{(a)}\big)^{k+1} u
\end{equation*}

\textbf{Step 2:} Next we use monotonicity of the operator $T^{(n)}$ and iteration to get the desired upper bound.

Define the composite operator $W \doteq \big(T^{(n)}\big)^a$. Because $T^{(n)}$ is an order-preserving (monotone) $\gamma^n$-contraction on the complete metric space $(\mathcal{V}, \Vert{}\cdot\Vert{}_\infty)$, the operator $W$ is a monotone $\gamma^{a\cdot n}$-contraction mapping.
Moreover, the unique fixed point of $W$ is $V_{\tau,\pi}^{(n)}$, because $T^{(n)} V_{\tau,\pi}^{(n)} = V_{\tau,\pi}^{(n)}$. Applying the operator at the fixed point $a$ more times stays at the fixed point, so  $W V_{\tau,\pi}^{(n)} = V_{\tau,\pi}^{(n)}$.

Because $W$ is order-preserving, we can apply it repeatedly to both sides of $V_{\tau,\pi}^{(a \cdot n)} \le W V_{\tau,\pi}^{(a \cdot n)}$ without flipping the inequality
\begin{equation*}
V_{\tau,\pi}^{(a \cdot n)} \le W V_{\tau,\pi}^{(a \cdot n)}\le W^2 V_{\tau,\pi}^{(a \cdot n)} \le \dots \le W^k V_{\tau,\pi}^{(a \cdot n)}
\end{equation*}
Since $W$ is a contraction mapping, the sequence $W^k V_{\tau,\pi}^{(a \cdot n)}$ converges uniformly to the unique fixed point of $W$ as $k \to \infty$. Therefore
\begin{equation*}
V_{\tau,\pi}^{(a \cdot n)} \le \lim_{k\to\infty} W^k V_{\tau,\pi}^{(a \cdot n)} = V_{\tau,\pi}^{(n)} 
\end{equation*}
completing the proof.
\end{proof}

\subsection{Convergence under policy iteration}\label{app_theory_pi}

In this section, we consider policy iteration using n-step expectile values. This is trickier than the standard setting with $\tau = 0.5$, because of the nonlinearity in the expectile. 
Policy iteration algorithms with the n-step expectile action-values can follow a similar alternating approach: we estimate $Q^{(n)}_{\tau, \pi}$ and then pick a new policy that is greedier wrt to $Q^{(n)}_{\tau, \pi}$ (puts higher probability on higher-valued actions). We want to show convergence to the optimal policy under a general class of greedification updates for $\pi$. 
One general condition that allows us to show convergence is given in the below definition. 
\begin{definition}[Monotonic Expectile Greedification]\label{def_exp_greedy}
Given a policy $\pi$ a new policy $\pi'$ is a \textbf{monotonic expectile greedy} policy if it satisfies 
\begin{align}\label{eq_monotonic}
(T_{\tau, \pi'}^{(n)} V_{\tau,\pi}^{(n)} (s) \ge V_{\tau,\pi}^{(n)}(s) 
\end{align}
for all $s$, with strict improvement if the inequality is strict for at least one $s$. 
\end{definition}
The first term $(T_{\tau, \pi'}^{(n)} V_\pi)$ is the expectile value for following policy $\pi'$ for $n$ steps and then bootstrapping from $V_\pi$ (as if following $\pi$ afterwards). We will see that this condition allows us to show $\pi'$ has a higher expectile value than $\pi$. 
Interestingly, we can see this is a generalization of the standard condition for expected returns ($\tau = 0.5$): $\pi'$ gives policy improvement if $\sum_{a \in \mathcal{A}} \pi'(a|s) Q_\pi(s,a) \ge \sum_{a \in \mathcal{A}} \pi(a|s) Q_\pi(s,a) = V_\pi(s)$. This is because, for $\tau = 0.5$, we have $(T_{\pi'} V_\pi)(s) = \sum_{a \in \mathcal{A}} \pi'(a|s) [ r(s,a) + \gamma \mathbb{E}[V_\pi(s')|s,a] = \sum_{a \in \mathcal{A}} \pi'(a|s) Q_\pi(s,a)$. We do not typically write the policy improvement condition using the Bellman operator $T_{\pi'}$, but it is equivalent. 

The monotonic expectile greedy condition for $\tau \neq 0.5$, however, is harder to interpret since it relies on the stochastic returns $G_s^{(n)}$ rather than the learned action-values $Q_{\tau, \pi}^{(n)}(s,a)$. In general, the condition usually used for expected returns---namely that $\sum_a \pi'(a|s) Q_\pi(s,a) \ge \sum_a \pi(a|s) Q_\pi(s,a)$---does not work for expectile values. Ideally, we would show that standard greedification approaches on $Q_{\tau, \pi}^{(n)}(s,a)$ produce a monotonic expectile greedy policy, even if they are not directly set to satisfy the condition. Unfortunately, 
this condition 
is not always satisfied by the standard greedification approaches, like softmax, $\epsilon$-greedy or $\epsilon$-greedy with a decay. We hypothesize it holds under using an $\epsilon$-greedy policy with additional assumptions on the size of $\epsilon$ and gap between action-values, but leave this verification for future work.

We assume a compact policy space $\Pi$ that can be a subset of all possible policies (e.g, it could be the set of soft-policies which have minimum probability of $\epsilon/|\mathcal{A}|$ on each action). For our convergence result, we need the policy class to be closed under the greedification operation. 

\begin{assumption}\label{assump_pclass}
The policy class $\Pi$ is compact and is closed under the greedification operator: for $\pi \in \Pi$, the policy greedification operator $\pi' = T(\pi)$ satisfies Definition \ref{def_exp_greedy} and has $\pi' \in \Pi$. 
\end{assumption}

\begin{theorem}\label{thm_main}
Assume the MDP has bounded rewards and that Assumption \ref{assump_pclass} is satisfied. Assume we start from any initial policy $\pi_0 \in \Pi$ and iteratively compute monotonic expectile greedy policies $\pi_{t+1} = T(\pi_t)$.  
Then $\pi_t$ as $t \rightarrow \infty$ converges to the optimal expectile policy $\pi^*_\tau \in \Pi$ (i.e., $V^{(n)}_{\tau, \pi^*_\tau}(s) \ge V^{(n)}_{\tau, \pi}(s)$ for all $s$ for any $\pi \in \Pi$. 
\end{theorem}
\begin{proof}
    At time step $t$, we have some policy $\pi_t$. By Theorem \ref{thm_n_fixed}, we know that $T^{(n)}_{\tau, \pi'}$ is a monotone contraction with unique fixed point $V^{(n)}_{\tau,\pi'}$. Therefore, $T^{(n)}_{\tau, \pi'}$ is order-preserving. We assumed that $\pi'$ is a monotonic expectile greedy policy, and so $ V^{(n)}_{\tau,\pi} \le T_{\tau, \pi'}^{(n)} V^{(n)}_{\tau,\pi}$. We can apply $T^{(n)}_{\tau, \pi'}$ repeatedly to both sides, and preserve the order: 
\begin{equation*}
V^{(n)}_{\tau,\pi} \le T_{\tau, \pi'}^{(n)} V^{(n)}_{\tau,\pi} \le (T_{\tau, \pi'}^{(n)})^2 V^{(n)}_{\tau,\pi} \le \ldots \le (T_{\tau, \pi'}^{(n)})^k V^{(n)}_{\tau,\pi} \le \ldots
\end{equation*}
Because $T^{(n)}_{\tau, \pi'}$ is a contraction, $(T_{\tau, \pi'}^{(n)})^k V^{(n)}_{\tau,\pi} \rightarrow V^{(n)}_{\tau,\pi'}$ as $k \rightarrow \infty$. Taking this limit gives $V^{(n)}_{\tau,\pi} \le V^{(n)}_{\tau,\pi'}$.

    Iteratively applying this monotonic policy improvement must eventually stop, because $V^{(n)}_{\tau, \pi}$ is upper bounded because the rewards are bounded and the value function space is bounded. Therefore, this procedure will stop at some $\pi^*_\tau \in \Pi$ that has the largest $V^{(n)}_{\tau, \pi^*_\tau(s)}$. 
\end{proof}

\begin{corollary}\label{cor_optimal}
If the MDP is deterministic (deterministic transitions and rewards) and $\Pi$ includes all deterministic policies, including the optimal policy $\pi^*$ (highest expected return), then $\pi^*_\tau = \pi^*$.  
\end{corollary}
\begin{proof}
    In a deterministic environment with $\Pi$ containing deterministic policies, $\pi^*_\tau$ will be deterministic (converge to following a deterministic optimal path) and the set of returns is composed of this one path. Therefore, we can also conclude that $Q^{(n)}_{\tau, \pi_\tau^*}(s,a) = Q_{\tau=0.5,\pi_\tau^*}(s,a)$ because the mean and expectiles are the same for a single return. If $\pi^*_\tau$ were not optimal, then we could greedify to get a better policy $\pi'(s)$, which would contradict the fact that there is no policy $\pi'$ that has a strictly better expectile action-value.   
\end{proof}

Note that if we used a strictly greedy policy in this deterministic setting, there is only one trajectory (no stochasticity) and the expectile operator is equivalent to the mean operator. Therefore, the above policy improvement result is trivially true if we use greedy policies, which is why we opt to show the result more generally for potentially stochastic policies. 

The above corollary shows that doing policy improvement with expectile action-values eventually converges to the optimal policy in a deterministic setting. It does not, however, guarantee monotonic improvement in terms of the expected return. In fact, this procedure does not have monotonic improvement, because greedifying wrt the expectile action-values can actually produce a policy with reduced expected returns. We show this in the next proposition. %
\begin{proposition}
There exists a deterministic MDP and $\tau > 0.5$ where $\pi_1, \pi_2$ satisfy the conditions of Theorem \ref{thm_main} and we have $V_{\tau, \pi_1}(s) \le V_{\tau, \pi_2}(s)$ for all $s$, but where $V_{\pi_1}(s) > V_{\pi_2}(s)$ for some $s$.
\end{proposition}
\begin{proof}
Define the following deterministic MDP with two states $s_0$ (the start state) and $s_1$, with two actions $a_1, a_2$. At $s_0$, taking action $a_1$ gives a reward of 5 and terminates. Taking action $a_2$ leads to $s_1$ and a reward of $0$.
From $s_1$, taking action $a_1$ gives a reward of $-10$ and terminates. Taking action $a_2$ gives a reward of $100$ and terminates. 

Define $\pi$ and $\pi'$ to be the same at $s_1$: $\pi(a_1\mid s_1)=0.9,\ \pi(a_2\mid s_1)=0.1$. They differ only in the first decision at $s_0$:
$\pi(a_1\mid s_0)=1$ (always take $a_1$)
and $\pi'(a_2\mid s_0)=1$ (always take $a_2$).

Now let us examine the returns under $\pi$. $G_{s_0,a_1}$ is the point mass $5$ (deterministic). $G_{s_0,a_2}$ has returns $-10$ w.p. $0.9$ and $+100$ w.p. $0.1$. Setting $\tau = 0.9$ gives 
\begin{align*}
Q_{0.9,\pi}(s_0,a_1)=5,\qquad Q_{0.9,\pi}(s_0,a_2)=45,\qquad V_{0.9,\pi}(s_0)=5\ (\text{$a_1$ always taken from $s_0$})
\end{align*}
Notice also that
\begin{align*}
\psi_{0.9}(G_{s_0,a_1} -V_{0.9,\pi}(s_0)) &= \psi_{0.9}(5-5)=0\\ 
\psi_{0.9}(G_{s_0,a_2} -V_{0.9,\pi}(s_0))&=0.9\psi_{0.9}(-10-5)+0.1\psi_{0.9}(100-5)=14.4
\end{align*}
We can see that $\pi'$ satisfies the monotonic expectile greedy condition (where all probability is on action $a_2$ from $s_0$)
\begin{align*}
\sum_{a} \pi'(a|s_0)\psi_{0.9}(G_{s_0,a} -V_{0.9,\pi}(s_0)) = \psi_{0.9}(G_{s_0,a_2} -V_{0.9,\pi}(s_0)) = 14.4 > 0
\end{align*}
But when we check the expected return, $\pi'$ is worse:  
\begin{align*}
V_{0.5,\pi}(s_0)=5 > V_{0.5,\pi'}(s_0)=0.9(-10)+0.1(100)=1 
\end{align*}
\end{proof}

\section{Algorithm Specification}\label{app_algspec}

Here we present the pseudocode for the Endpoint replay method, as introduced in \ref{sec:end-point}. $\pi
_{\epsilon}(s)$ represents using the $\epsilon$-greedy exploration strategy with $\pi(s) = \arg \max_a (Q_{\theta}(s,a))$. $\ell_{\tau}$ represents the Expectile loss function defined for Expectile Sarsa in \ref{sec-expectile-sarsa}. We set the value of $\tau$ to $0.7$ for Endpoint replay across all experiments. As previously mentioned, the discount factor $\gamma$ is set to 0 in termination per \citep{white2017unifying}. 

\begin{algorithm}[htbp]
\caption{Endpoint Replay}
\label{alg:endpoint_replay}
\begin{algorithmic}[1]
\State \textbf{Initialize} online network $Q_{\theta}$, target network $Q_{\theta^{-}}$
\State \textbf{Initialize} recency buffer $\mathcal{D}_{r}$, coreset buffer $\mathcal{D}_{c}$, lag buffer $\mathcal{D}_{\text{lag}}$
\State \textbf{Initialize} capacities $N_{\text{recency}}$, $N_{\text{coreset}}$  target refresh frequency $N_{\text{target}}$, warmup steps $N_{\text{warmup}}$
\State Observe initial state $s$, choose action $a \sim \pi_{\epsilon}(s)$

\For{$t = 1, 2, \dots, T$}
    \State Execute action $a$, observe reward $r$, next state $s'$, and termination-aware discount factor $\gamma$.
    \State Choose next action $a' \sim \pi_{\epsilon}(s')$ 
    \State Add transition $(s, a, r, s', a', \gamma)$ to $\mathcal{D}_{r}$
    
    \Statex \Comment{\textbf{Buffer Eviction \& Coreset Summarization}}
    \If{$|\mathcal{D}_{r}| > N_{\text{recency}}$}
        \State Pop oldest transition $(s_e, a_e, r_e, s'_e, a'_e, \gamma)$ from $\mathcal{D}_{r}$ and append to $\mathcal{D}_{\text{lag}}$
        
        \If{$|\mathcal{D}_{\text{lag}}| == n$ \textbf{or} termination}
            \State Let $k \gets |\mathcal{D}_{\text{lag}}|$
            \State Let $(s_0, a_0)$ be the state and action from the \textit{first} transition in $\mathcal{D}_{\text{lag}}$
            \State Compute $k$-step return: $g \gets \sum_{i=0}^{k-1} \gamma^i r_{e-k+1+i}$
            \State Pop the first transition from $\mathcal{D}_{\text{lag}}$
            \State Add summary transition $(s_0, a_0, g, s'_e, a'_e, \gamma)$ to $\mathcal{D}_{c}$
        \EndIf
    \EndIf
    
    \Statex \Comment{\textbf{Network Updates}}
    \If{$t > N_{\text{warmup}}$}
        \State Sample batch $\mathcal{M}_r$ of size $B_r$ from $\mathcal{D}_{r}$, and batch $\mathcal{M}_c$ of size $B_c$ from $\mathcal{D}_{c}$
        
        \Statex \hspace{\algorithmicindent} \Comment{Compute Targets}
        \For{\textbf{each} $(s, a, r, s', \gamma) \in \mathcal{M}_r$}
            \State $y_{r} \gets r + \gamma Q_{\theta^{-}}(s', \arg\max_{a^*} Q_{\theta}(s', a^*))$
            \Comment{DDQN target}
        \EndFor
        
        \For{\textbf{each} $(s_e, a_e, g, s_{\text{end}}, a_{\text{end}}, \gamma) \in \mathcal{M}_c$}
            \State $y_{c} \gets g + \gamma^k  Q_{\theta^{-}}(s_{\text{end}}, a_{\text{end}})$
            \Comment{SARSA target}
        \EndFor
        
        \Statex \hspace{\algorithmicindent} \Comment{Compute Losses \& Optimize}
        \Statex \hspace{\algorithmicindent}
        \State $\mathcal{L}_{r}(\theta) \gets \frac{1}{|\mathcal{M}_r|}\sum_{\mathcal{M}_r} \left( y_{r} - Q_{\theta}(s, a) \right)^2$
        \State $\mathcal{L}_{c}(\theta) \gets \frac{1}{|\mathcal{M}_c|}\sum_{\mathcal{M}_c} \ell_\tau \big( y_{c} - Q_{\theta}(s_e, a_e) \big)$
        
        \State Update $\theta$ by minimizing $\mathcal{L}_{r}(\theta) + \mathcal{L}_{c}(\theta)$ via gradient descent 
    \EndIf

        \algstore{endpoint_break}
\end{algorithmic}
\end{algorithm}

\newpage

\begin{algorithm}[H]
\caption{Endpoint Replay (Continued)}
\begin{algorithmic}[1]
    \algrestore{endpoint_break}
    
    \If{$t \pmod{N_{\text{target}}} == 0$} 
        \State $\theta^{-} \gets \theta$
    \EndIf
    \If{$termination$}
        \State Observe new initial state $s$, choose action $a \sim \pi_{\epsilon}(s)$
    \Else
        \State $s \gets s'$, $a \gets a'$
    \EndIf
\EndFor
\end{algorithmic}
\end{algorithm}

\section{The Pinball Environment}

The Pinball domain \citep{konidaris2009skill,lo2024goal} is an episodic environment with simplified physics, where the objective is to navigate the ball to the goal, while avoiding obstacles, by applying force to the ball in all four cardinal directions. The observation space is a 4 dimensional vector of agent's position and velocity in the plane. There are 5 actions, applying force in cardinal directions and a fifth no-op action. The original version of the environment used a complex reward function making analysis difficult. We chose to follow \cite{lo2024goal}'s lead and simplify the reward function to be -1 per step. The discount factor is 0.99 and the environment terminates when the agent reaches the goal region or times out after 1000 steps.

\section{Selected Atari Games}

We selected 12 games from the Arcade Learning Environment \citep{bellemare2013arcade, machado2018revisiting} suite for our empirical evaluation. First we use the Atari-5 games \citep{aitchison2023atari}, selected to correlate with the aggregate score of all the Atari games. In addition, we also selected the following games, specifically chosen to present a diverse set of problem structures:

\begin{itemize}
    \item \textbf{Pong:} Selected to test learning under sparse, delayed rewards.
    \item \textbf{Breakout:} Chosen as a baseline for environments with smooth, continuous score changes.
    \item \textbf{Seaquest:} Included due to its complex, cyclic dynamics.
    \item \textbf{Centipede:} Selected to evaluate robustness against constantly changing input distributions.
    \item \textbf{Ms.\ Pac-Man:} Included to test adaptability to the high variance induced by unpredictable ghost movements.
    \item \textbf{Beam Rider:} Chosen for its somewhat drastic sector and environment changes.
    \item \textbf{Space Invaders:} Selected due to its combination of irreversible state changes and large, sparse rewards.
\end{itemize}

\newpage

\section{Pinball Experiment Details}

In Pinball we use a small DDQN agent as the base learning algorithm of all the methods. The large recency buffer baseline, whose hyperparameters are available in Table \ref{tab:pinball_hypers}, is tuned by sweeping the learning rate over $\{0.0005, 0.001, 0.002, 0.004, 0.008\}$ for 10 seeds and selecting the best average performance. All tested algorithms have the same shared hyperparameters except replay buffer size.

\begin{table}[htbp]
    \centering

    \begin{tabular}{ll}
        \toprule
        \textbf{Hyperparameter} & \textbf{Value} \\
        \midrule
        \addlinespace
        Exploration $\epsilon$ & $0.1$ \\
        Target Network Refresh Rate & $100$ \\
        Buffer Size & $10,000$ \\
        Warmup Steps & $1,000$ \\
        batchsize & $32$ \\
        Network & Two layer ReLU of size 32\\
        Optimizer & Adam \\
        Learning Rate & $0.002$ \\
        Adam $\beta_1$ & $0.9$ \\
        Adam $\beta_2$ & $0.999$ \\
        \bottomrule
    \end{tabular}

    \caption{Hyperparameters for the large buffer DDQN agent in the Pinball environment. Values are in terms of the number of interaction steps.}
    \label{tab:pinball_hypers}
\end{table}

\section{Atari Experiment Details}

In the Atari experiments we use DDQN's original architecture \citep{van2016deep} and use Dopamine's hyperparameters, selected to closely follow the original DQN paper \citep{mnih2015human}, but adopts some aspects of Rainbow such as the Adam optimizer \citep{kingma2015adam}. Similar to Pinball we present the list of hyperparameters for the large buffer baseline in Table \ref{tab:atari-hypers}.

\begin{table}[h]
    \centering
    \begin{tabular}{ll}
        \toprule
        \textbf{Parameter} & \textbf{Value} \\
        \midrule
        \addlinespace
        Final exploration $\epsilon_{\text{final}}$ & 0.01 \\
        Exploration Annealing Frames & 1,000,000 \\
        Buffer Size & 1,000,000 \\
        Warmup Frames & 80,000 \\
        batchsize & 32 \\
        Update Frequency & 16 Frames \\
        Target Network Refresh & 32,000 Frames\\
        Optimizer & Adam \\
        Learning Rate & $6.25 \times 10^{-5}$ \\
        $\beta_1$ & 0.9 \\
        $\beta_2$ & 0.999 \\
        $\epsilon_{\text{adam}}$ & $1.5 \times 10^{-4}$ \\
        \bottomrule
    \end{tabular}
     \caption{Hyperparameters for the large recency DDQN agent in the Atari environment.}
    \label{tab:atari-hypers}
\end{table}

We closely follow the conventional environment settings for the Arcade Learning Environment, adopting sticky actions \citep{machado2018revisiting} instead of initial no-ops. The complete list of environment settings is presented in Table \ref{tab:atari_setting}.

\begin{table}[H]
    \centering
    \begin{tabular}{ll}
        \toprule
        \textbf{Setting} & \textbf{Value} \\
        \midrule
        Episode Cutoff & 108,000 Frames \\
        Repeat Action Probability & 0.25 \\
        Frame Skip & 4 \\
        Frame Stack & 4 \\
        Action Set & Limited \\
        Image size & $84 \times 84$ \\
        Grayscale & Yes \\
        \bottomrule
    \end{tabular}
    \caption{Environment settings for the Atari experiments.}
    \label{tab:atari_setting}
\end{table}

In the paper we present aggregated performance of agents in Atari. We use min--max normalization to aggregate the results. In each game, we map the performance of worst performing run among the algorithms of interest to 0.0, while mapping the best run to 1.0. Then we report the averaged normalized performance.

\section{Sign Test}
We use the sign test as our statistical significance test for comparing the
performances of agents in the Atari experiments. We tested 12 Atari games,
each with 10 seeds. To compare two agents $X$ and $Y$, we have in total 120
pairs of average returns $(x_i, y_i)$ for games and seeds. The pairs that have
no difference in $x_i$ and $y_i$ are omitted, and eventually we can have $n$
pairs.

We perform a one-sided test by asking if $X$ has a higher return than $Y$:
$X > Y$. Let $w$ be the number of pairs where $x_i > y_i$, and the null
hypothesis

\[
H_0: \text{Pr}(X>Y) = 0.5.
\]
Assuming that $H_0$ is true, then $W$ follows a binomial distribution
\[
W \sim \mathcal{B}(n, 0.5).
\]

The right-tail value is computed by
\[
p = \text{Pr}(W \ge w),
\]
which is the $p$-value for the alternative hypothesis
$H_1: \text{Pr}(X>Y) > 0.5$. We choose a significance level $\alpha = 0.05$. If $p \le \alpha$, we reject the null hypothesis and conclude that agent $X$ is better than agent $Y$ by achieving statistically significantly higher returns.

\section{Atari Results in Individual Games}
Here we present a complete set of the Atari experiments in individual games. The first set of plots are bar plots of average return across 10 seeds for each algorithm in each game. The next set of plots are learning curves showing average performance as well as 95\% bootstrap confidence intervals.

\begin{figure}
    \centering
    \includegraphics[width=1.0\linewidth]{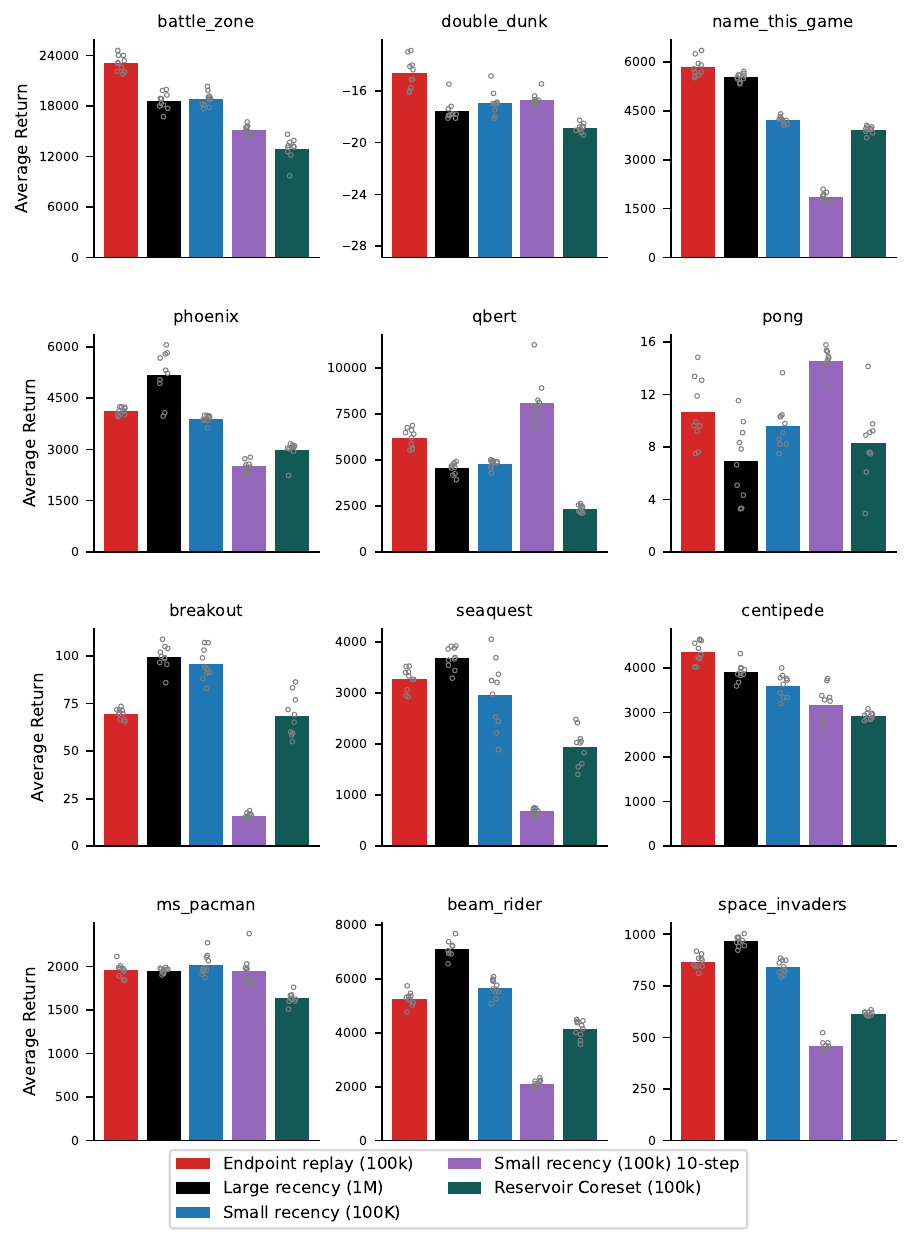}
    \caption{Performance of Endpoint replay and other baselines in individual Atari games in the 100k buffer setting. Bars represent average performance across 10 seeds and dots represent performance of each seed. Endpoint replay (100K) outperforms Small recency (100K) by a significant margin according to the Sign test (better in 82 out of 120 paired seeds across 12 games).}
    \label{fig:atari_100k}
\end{figure}

\begin{figure}
    \centering
    \includegraphics[width=1.0\linewidth]{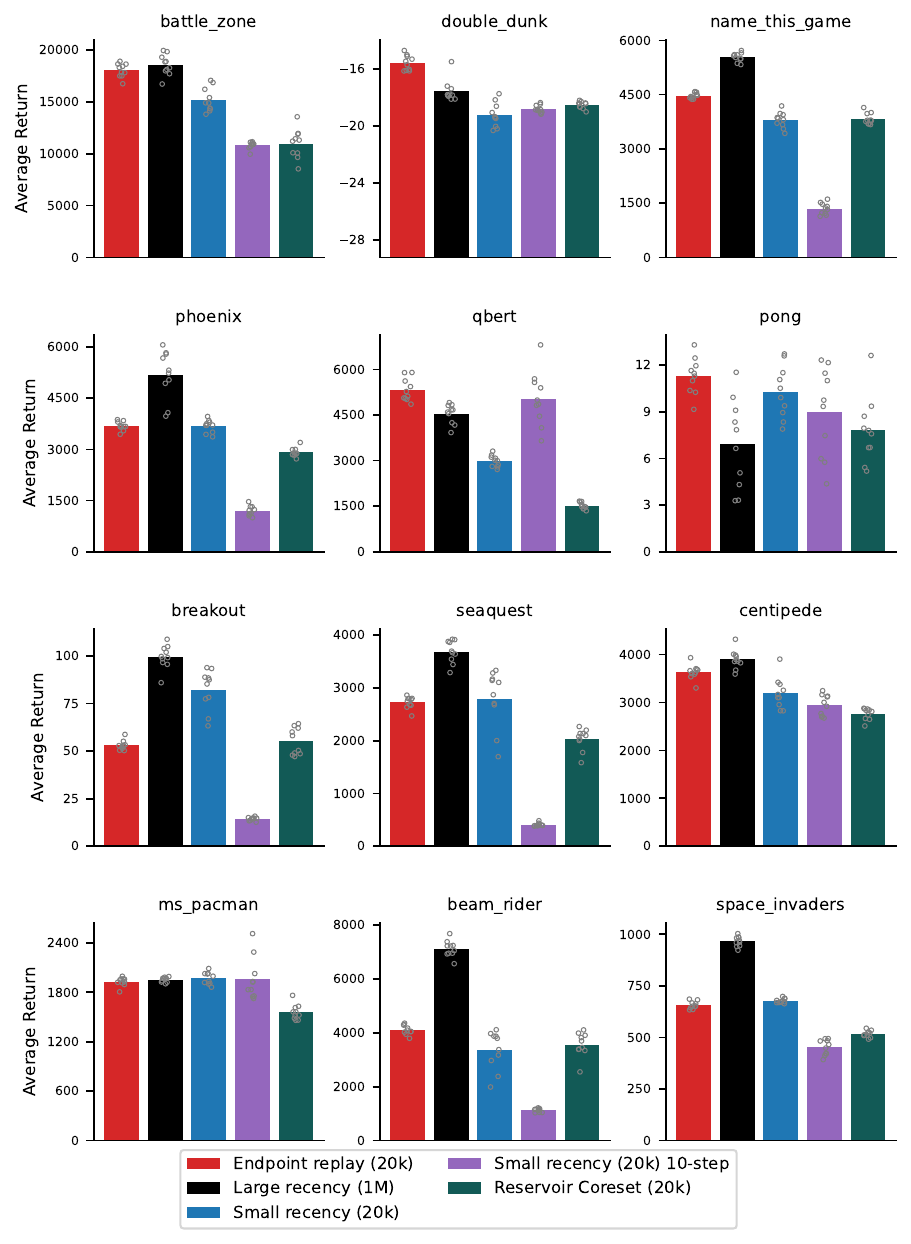}
    \caption{Performance of Endpoint replay and other baselines in individual Atari games in the 20k buffer setting. Bars represent average performance across 10 seeds and dots represent performance of each seed. Endpoint replay (20K) outperforms Small recency (20K) by a significant margin according to the Sign test (better in 80 out of 120 paired seeds across 12 games).}
    \label{fig:atari_20k}
\end{figure}

\begin{figure}
    \centering
    \includegraphics[width=1.0\linewidth]{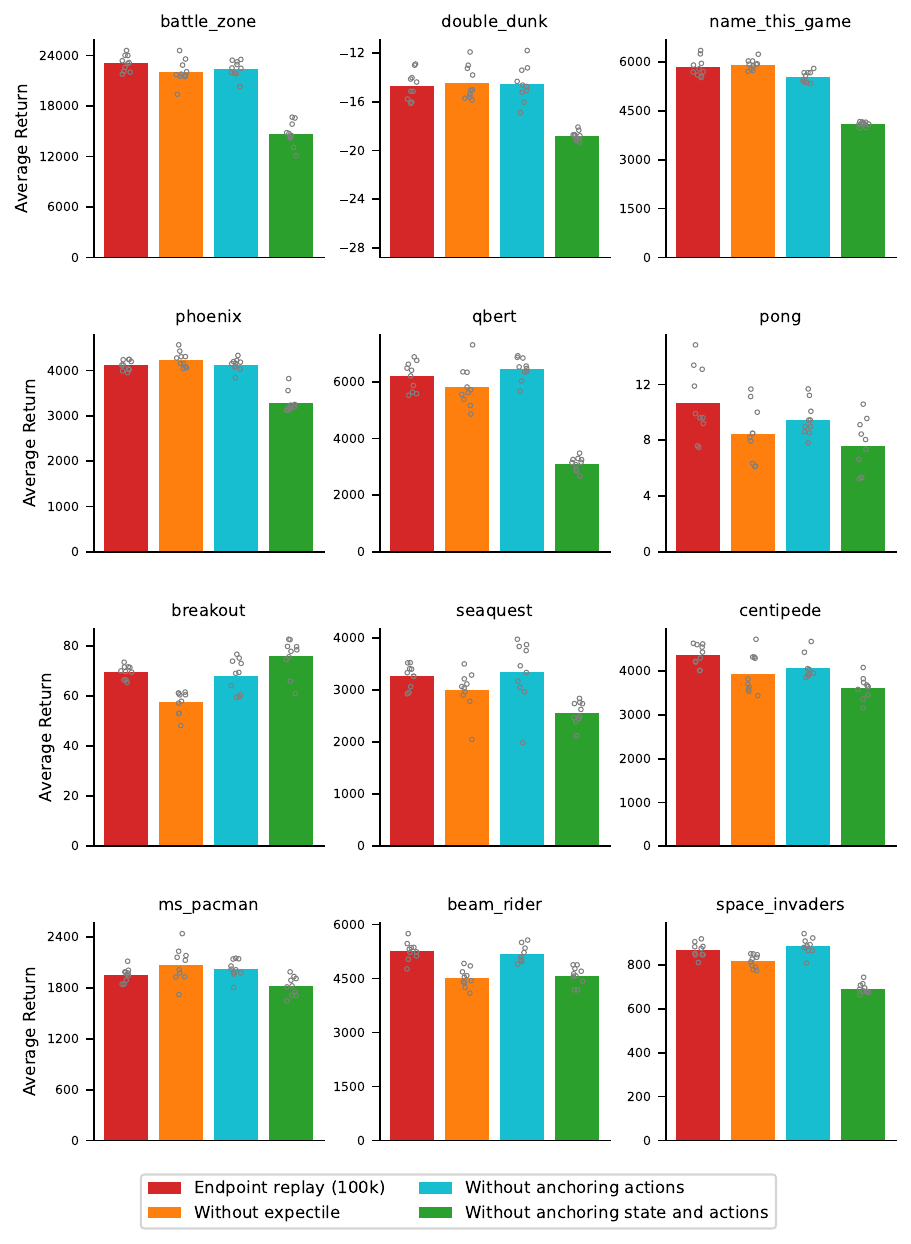}
    \caption{Performance of Endpoint replay and its ablations in individual Atari games in the 100k buffer setting. Bars represent average performance across 10 seeds and dots represent performance of each seed. Within 120 paired seeds across 12 games, Endpoint replay (100K) outperforms Without expectile in 79 seeds, Without anchoring actions in 67 seeds, and Without anchoring state and actions in 109 seeds. According to the Sign test ($\alpha=0.05$), Endpoint replay (100K) outperforms Without expectile and Without anchoring state and actions by a significant margin.}
    \label{fig:ablations_atari_100k}
\end{figure}

\begin{figure}
    \centering
    \includegraphics[width=1.0\linewidth]{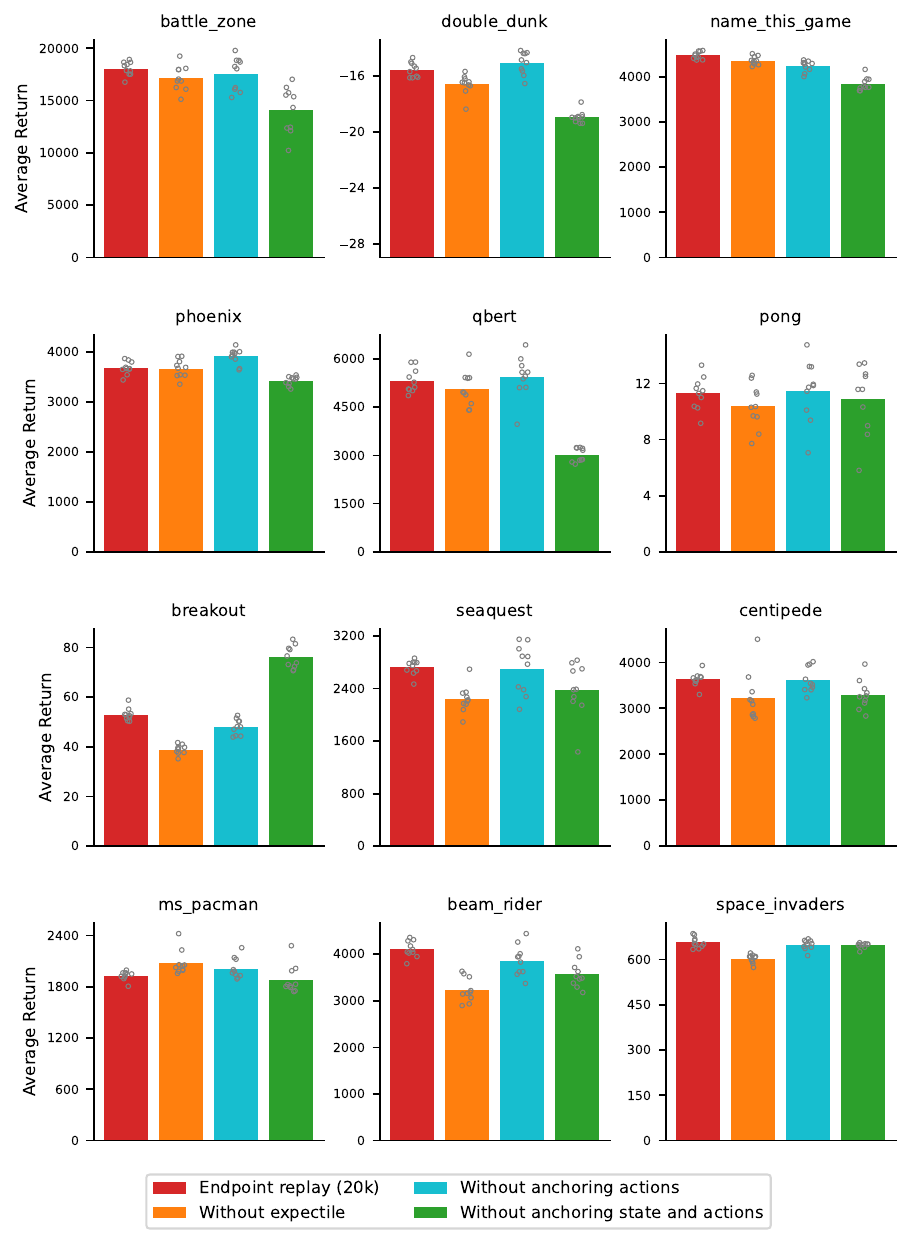}
    \caption{Performance of Endpoint replay and its ablations in individual Atari games in the 20k buffer setting. Bars represent average performance across 10 seeds and dots represent performance of each seed. Within 120 paired seeds across 12 games, Endpoint replay (20K) outperforms Without expectile in 93 seeds, Without anchoring actions in 62 seeds, and Without anchoring state and actions in 96 seeds. According to the Sign test ($\alpha=0.05$), Endpoint replay (20K) outperforms Without expectile and Without anchoring state and actions by a significant margin.}
    \label{fig:ablations_atari_20k}
\end{figure}

\begin{figure}
    \centering
    \includegraphics[width=1.0\linewidth]{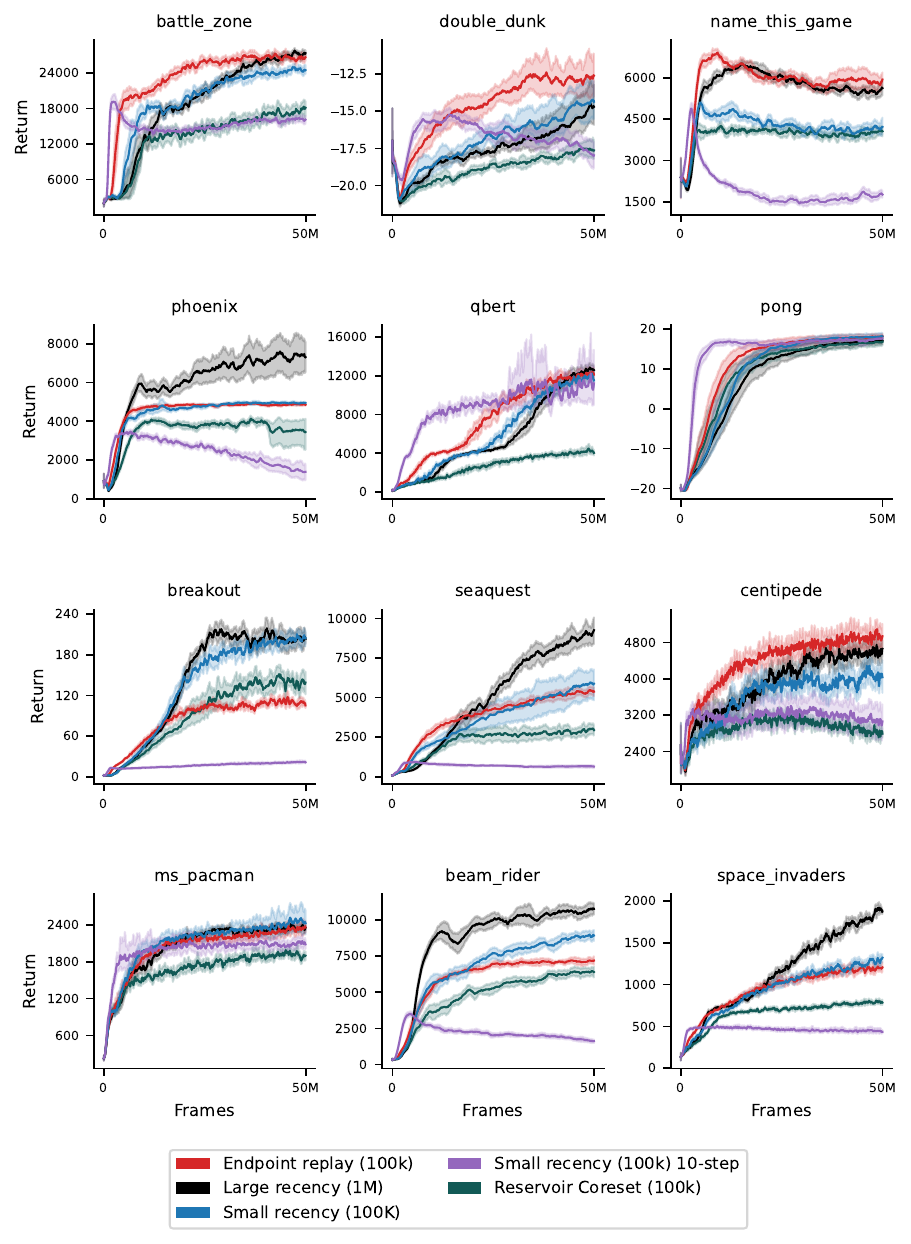}
    \caption{Performance of Endpoint replay and other baselines in individual Atari games in the 100k buffer setting. Return of each run is smoothed over the last 100 episodes then aggregated across 10 seeds with shaded regions showing 95\% bootstrap CI.}
    \label{fig:atari_100k_curves}
\end{figure}

\begin{figure}
    \centering
    \includegraphics[width=1.0\linewidth]{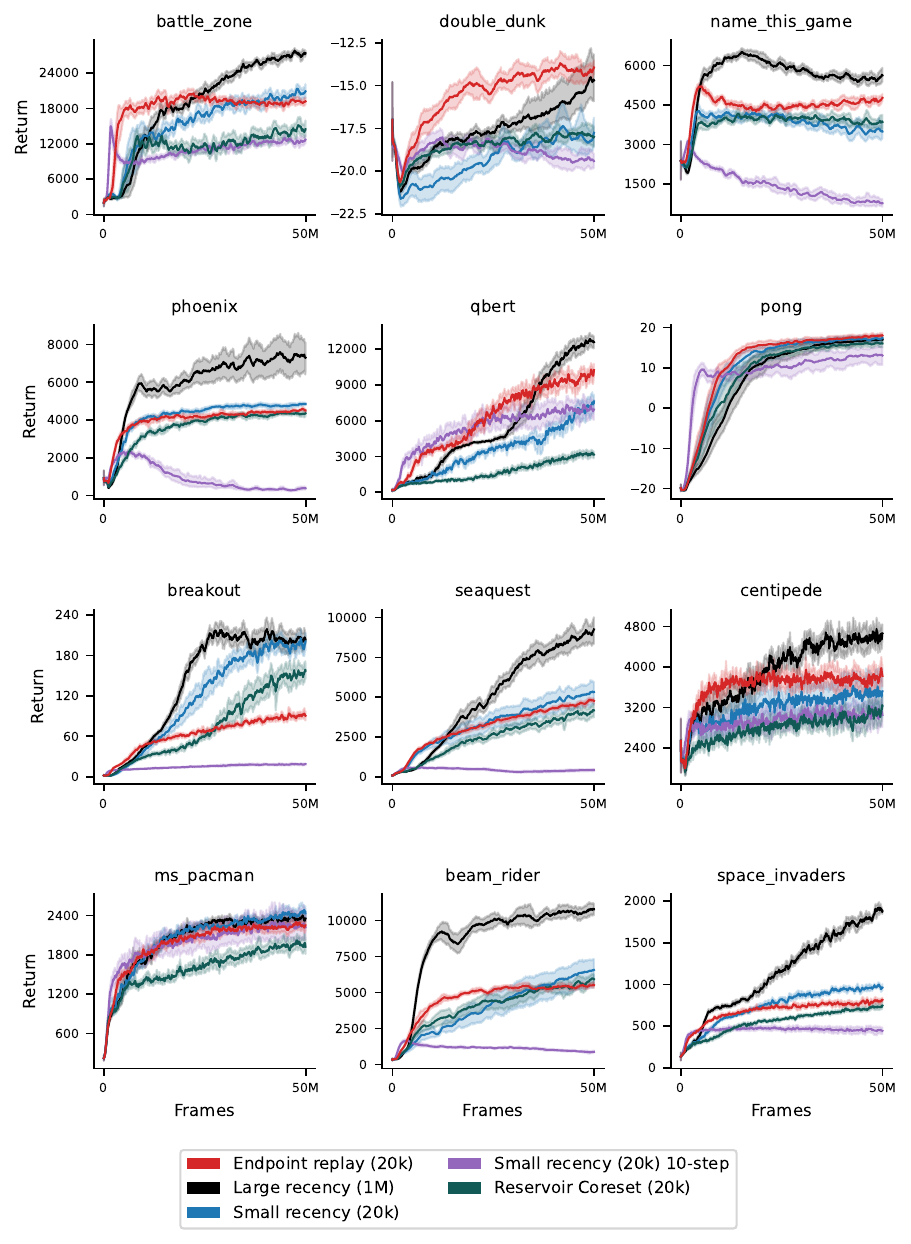}
    \caption{Performance of Endpoint replay and its ablations in individual Atari games in the 100k buffer setting. Return of each run is smoothed over the last 100 episodes then aggregated across 10 seeds with shaded regions showing 95\% bootstrap CI.}
    \label{fig:atari_20k_curves}
\end{figure}

\begin{figure}
    \centering
    \includegraphics[width=1.0\linewidth]{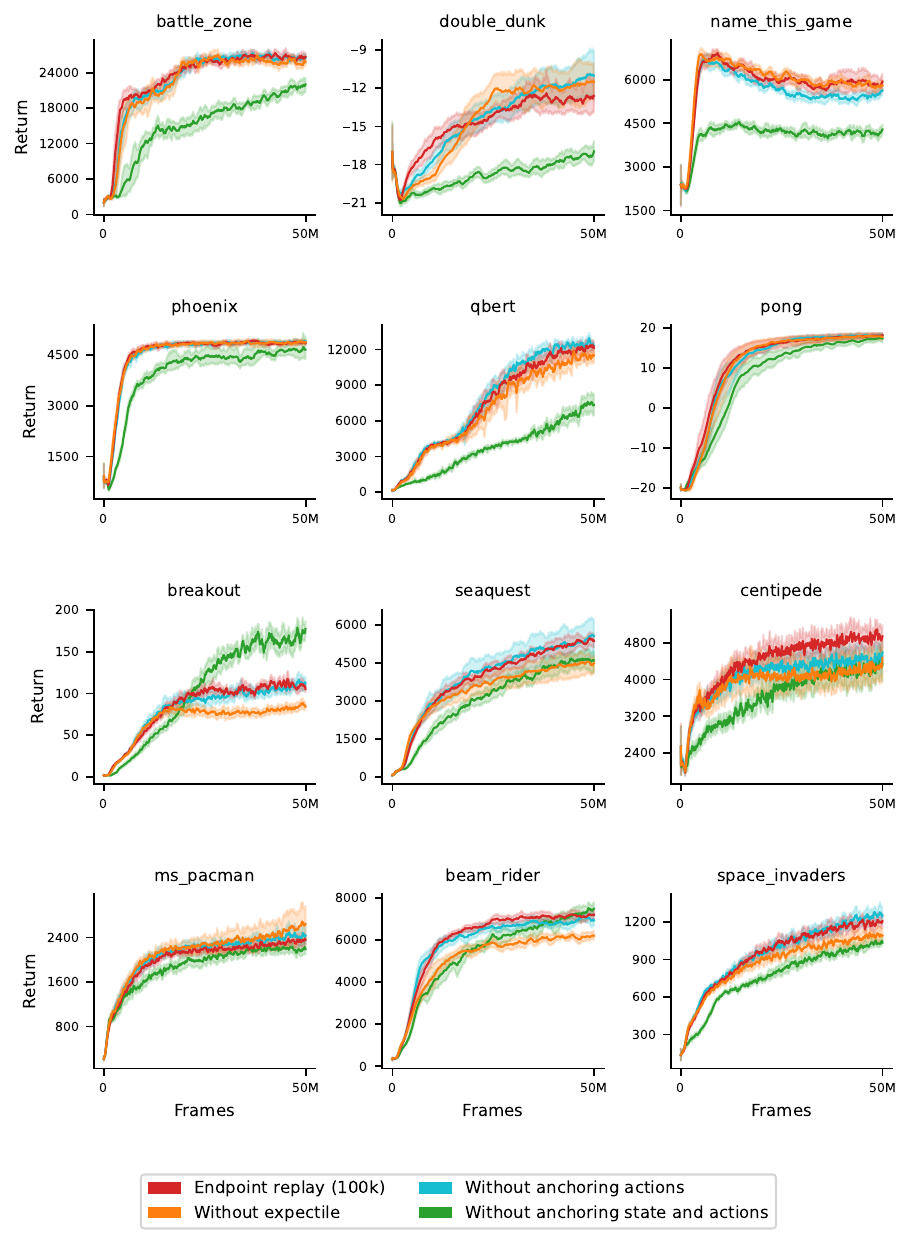}
    \caption{Performance of Endpoint replay and its ablations in individual Atari games in the 100k buffer setting. Return of each run is smoothed over the last 100 episodes then aggregated across 10 seeds with shaded regions showing 95\% bootstrap CI.}
    \label{fig:ablations_atari_100k_curves}
\end{figure}

\begin{figure}
    \centering
    \includegraphics[width=1.0\linewidth]{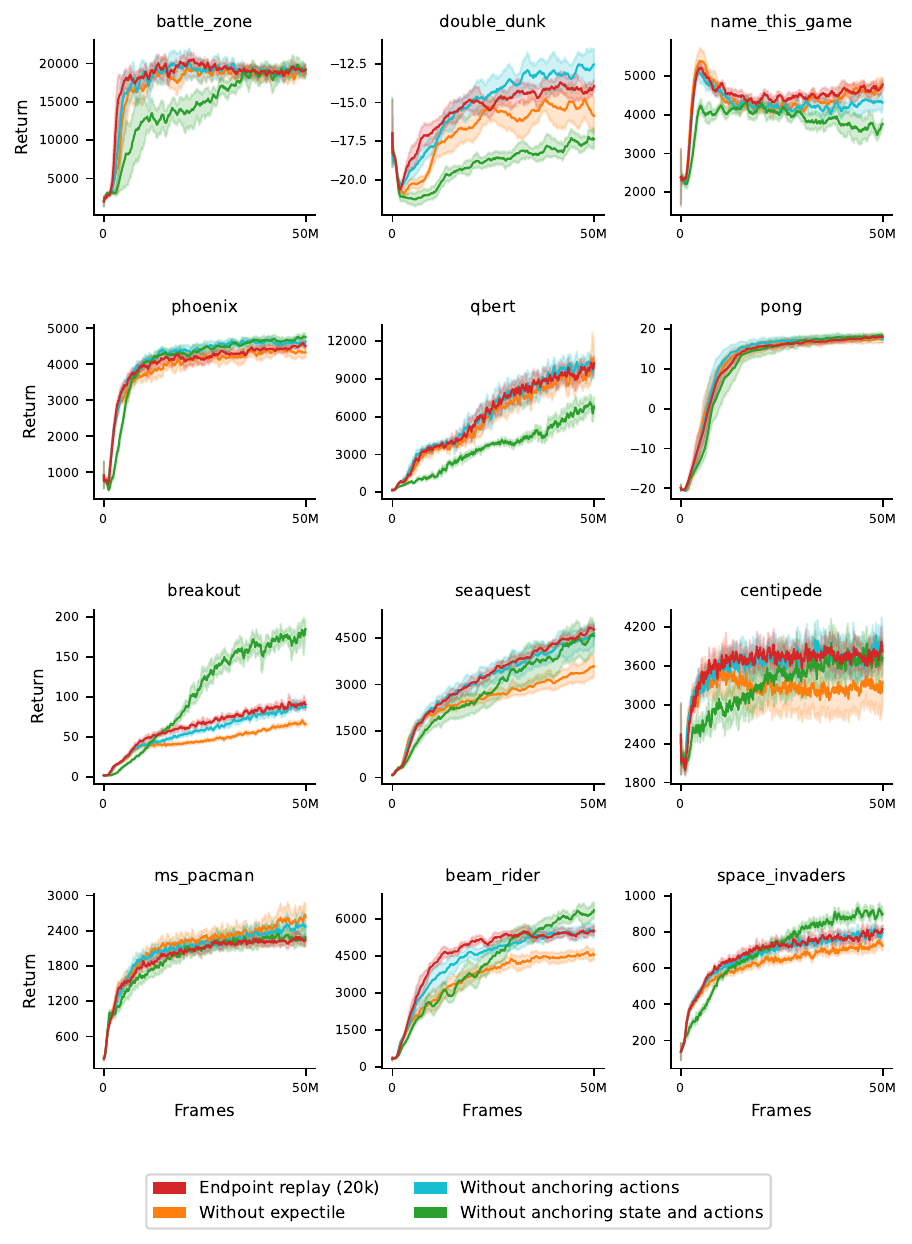}
    \caption{Performance of Endpoint replay and its ablations in individual Atari games in the 20k buffer setting. Return of each run is smoothed over the last 100 episodes then aggregated across 10 seeds with shaded regions showing 95\% bootstrap CI.}
    \label{fig:ablations_atari_20k_curves}
\end{figure}

\end{document}